\documentclass{article}
\usepackage[preprint]{neurips_2025}
\usepackage{hyperref}
\usepackage{url}
\usepackage{booktabs}       % professional-quality tables
\usepackage{amsfonts}       % blackboard math symbols
\usepackage{nicefrac}       % compact symbols for 1/2, etc.
\usepackage{microtype}      % microtypography
\usepackage{xcolor}   % colors
\usepackage{graphicx}
\usepackage{wrapfig}
\usepackage{multirow}
\usepackage{epsfig}
\usepackage{marvosym}
\usepackage{color}
\usepackage{threeparttable}
\usepackage{setspace}
\usepackage{amsthm}
\usepackage{verbatim}
\usepackage{geometry}
\usepackage{booktabs}
\usepackage{xcolor}
\usepackage{colortbl}
\usepackage{bibentry}
\usepackage{amsfonts}
\usepackage{multirow}
\usepackage{graphicx}
\usepackage{epsfig}
\usepackage{color}
\usepackage{epstopdf}
\usepackage{rotating}
\usepackage{caption}
\usepackage{xspace}
\usepackage{float}
\usepackage{multicol}
\usepackage{graphicx}
\usepackage{tabularx}
\usepackage{balance}
\usepackage{microtype}
\usepackage{graphicx}
\usepackage{booktabs} % for professional tables
\usepackage{multirow}
\usepackage{amsmath}
\usepackage{pifont}
\usepackage{mathtools}
\usepackage{etoolbox}
\usepackage{cases}
\usepackage{enumitem}
\usepackage{xcolor}
\usepackage{subcaption}
\usepackage[ruled,vlined]{algorithm2e}
\newtheorem{theorem}{Theorem}
\SetKwInput{KwIn}{Input}
\SetKwInput{KwOut}{Output}

\definecolor{backgreen}{RGB}{213, 232, 212}
\newcommand{\backg}{\cellcolor{backgreen!50}}
\usepackage{color}
\definecolor{backred}{RGB}{255, 190, 190}

\newcommand{\boldres}[1]{{\textbf{\textcolor{red}{#1}}}}
\newcommand{\secondres}[1]{{\underline{\textcolor{blue}{#1}}}}
\definecolor{brown}{RGB}{139,64,0}
\newcommand{\ourname}{{DeMa}}
\newcommand{\SSD}{{Mamba-SSD}}
\newcommand{\DALA}{{Mamba-DALA}}
\definecolor{darkgreen}{rgb}{0.0,0.5,0.0}
\newcommand{\cmark}{\textcolor{darkgreen}{\ding{51}}}
\newcommand{\xmark}{\textcolor{red}{\ding{55}}}
\usepackage{cases}
 
\title{DeMa: Dual-Path Delay-Aware Mamba for Efficient Multivariate Time Series Analysis}

\author{
		Rui An\textsuperscript{1,2,}\thanks{Equal Contribution} ,
		Haohao Qu\textsuperscript{2,}\footnotemark[1], 
        Wenqi Fan\textsuperscript{2}\textsuperscript{\Letter}, 
		Xuequn Shang\textsuperscript{1}\textsuperscript{\Letter}, 
        Qing Li\textsuperscript{2}\\
    \textsuperscript{1}Northwestern Polytechnical University,
    \textsuperscript{2}The Hong Kong Polytechnic Univeristy\\
    \{rui77.an, haohao.qu\}@connect.polyu.hk, \\wenqifan03@gmail.com, shang@nwpu.edu.cn, csqli@comp.polyu.edu.hk\\
}

\begin{document}

\maketitle

\begin{abstract}
Accurate and efficient multivariate time series (MTS) analysis is increasingly critical for a wide range of intelligent applications, including traffic forecasting, anomaly detection for industrial maintenance, and trajectory classification for health monitoring. Within this realm, Transformers have emerged as the predominant architecture due to their strong ability to capture pairwise dependencies. However, Transformer-based models suffer from quadratic computational complexity and high memory overhead, limiting their scalability and practical deployment for long-term, large-scale MTS modeling.
Recently, Mamba has emerged as a promising linear-time alternative with high expressiveness. Nevertheless, directly applying vanilla Mamba to MTS remains suboptimal due to three key limitations: (i) the lack of explicit cross-variate modeling, (ii) difficulty in disentangling the entangled intra-series temporal dynamics and inter-series interactions, and (iii) insufficient modeling of latent time-lag interaction effects. These issues constrain its effectiveness across diverse MTS tasks.
To address these challenges, we propose \textbf{\ourname{}}, a dual-path \textbf{De}lay-Aware \textbf{Ma}mba backbone for efficient and effective MTS analysis. \ourname{} preserves Mamba’s linear-complexity advantage while substantially improving its suitability for multivariate settings. Specifically, \ourname{} introduces three key innovations: (i) it decomposes the MTS context into intra-series temporal dynamics and inter-series interactions and learns them via two dedicated paths; (ii) it develops a temporal path with a \textit{Mamba-SSD} module to capture long-range dynamics within each series, accommodating variable-length inputs and enabling series-independent, parallel computation while maintaining linear complexity; and (iii) it designs a variate path with a \textit{Mamba-DALA} module that integrates delay-aware linear attention to model cross-variate dependencies, enhancing fine-grained, delay-sensitive dependency learning. Extensive experiments on five representative tasks, long- and short-term forecasting, data imputation, anomaly detection, and series classification, demonstrate that \ourname{} achieves state-of-the-art performance while delivering remarkable computational efficiency.
\end{abstract}

\section{Introduction}

Multivariate time series (\textbf{MTS}) are a crucial and fundamental data modality in both industrial and daily scenarios. They are collected at scale by physical and virtual sensors, which continuously record system measurements and encapsulate valuable information about the evolving dynamics of real-world systems~\cite{kong2025deep,liang2025itfkan}. 
Consequently, time series analysis serves as a cornerstone for understanding and predicting the behaviors of complex systems, enabling a wide range of intelligent applications, such as traffic flow forecasting for transportation scheduling~\cite{an2025damba,huo2023hierarchical}, missing data imputation for web stream processing~\cite{fang2020time}, anomaly detection for maintenance~\cite{chen2024lara,blazquez2021review}, and trajectory classification for health monitoring~\cite{qin2020imaging}.

\begin{figure*}[t]
    \centering
    \begin{subfigure}[b]{0.43\textwidth}
        \centering
        \includegraphics[width=\linewidth]{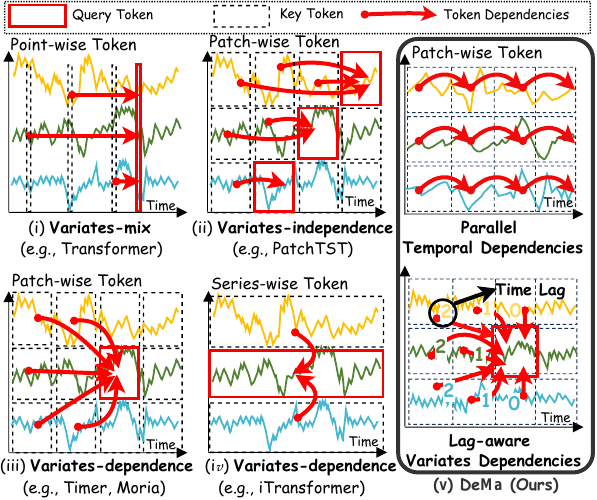}
        \caption{}
        \label{fig:motivation_a}
    \end{subfigure}
    \hfill
    \begin{subfigure}[b]{0.53\textwidth}
        \centering
        \resizebox{\linewidth}{!}{%
        \renewcommand{\arraystretch}{1.4}
            \begin{tabular}{ccccc}
            \toprule
            \multirow{2}{*}{Category} & Representative & Temporal & Variate & Computational \\
            & Model & Dependencies & Dependencies & Complexity \\
            \midrule
            \multirow{2}{*}{(i)} & \multirow{2}{*}{Transformer~\cite{vaswani2017attention}} 
            & \multirow{2}{*}{\cmark} & \multirow{2}{*}{\xmark} 
            & \multirow{2}{*}{$\mathcal{O}(\textcolor{red}{T^2}D)$} \\
            & & & & \\
            \midrule
            \multirow{2}{*}{(ii)} & PatchTST~\cite{PatchTST} 
            & \cmark & \xmark & $\mathcal{O}(N\textcolor{red}{L^2}D)$ \\
            & Timer~\cite{liutimer} 
            & \cmark & \xmark & $\mathcal{O}(N\textcolor{red}{L^2}D)$ \\
            \midrule
            \multirow{2}{*}{(iii)} & \multirow{2}{*}{Moirai~\cite{woounified}} 
            & \multirow{2}{*}{\cmark} & \multirow{2}{*}{\cmark} 
            & \multirow{2}{*}{$\mathcal{O}(\textcolor{red}{N^2L^2}D)$} \\
            & & & & \\
            \midrule
            \multirow{2}{*}{(iv)} & \multirow{2}{*}{iTransformer~\cite{liuitransformer}} 
            & \multirow{2}{*}{\cmark} & \multirow{2}{*}{\cmark} 
            & \multirow{2}{*}{$\mathcal{O}(\textcolor{red}{N^2}D)$} \\
            & & & & \\
            \midrule
            \multirow{2}{*}{(v)} & \multirow{2}{*}{\textbf{\ourname{}} (\textbf{Ours})} 
            & \multirow{2}{*}{\cmark} & \multirow{2}{*}{\cmark} & \backg \\
            & & & & \backg \multirow{-2}{*}{$\mathcal{O}(2NL\textcolor{red}{D^2})$} \\
            \bottomrule
            \end{tabular}%
        }
        \caption{}
        \label{fig:motivation_b}
    \end{subfigure}
    \caption{Dependency modeling paradigms and complexity comparison of representative multivariate time-series architectures. 
    (a) Illustration of tokenization schemes and dependency modeling strategies, including variate-mixing, variate-independence, and variate-dependence designs. 
    (b) Computational complexity comparison of representative models. 
    Here, $T$ denotes the lookback length, $N$ the number of variates, $L$ the token length, and $D$ the embedding dimension. In typical long-horizon settings, $T > L \gg N, D$. \textbf{\ourname{}} decouples temporal modeling from delay-aware cross-variate interactions, achieving linear-time complexity.}
    \label{fig:motivation}
\end{figure*}

Due to the complex and non-stationary nature of real-world systems, observed MTS often exhibit intricate and entangled patterns. Along the temporal axis, multiple variations, including long-term trends, seasonal cycles, and irregular fluctuations, are typically mixed and overlapped. Along the variate axis, correlated series interact through latent dependencies, where the evolution of one variable can influence, or be influenced by, others with time delays. Accordingly, a series of deep models have been proposed to capture \textit{temporal dependencies} and \textit{variate dependencies}~\cite{han2024capacity}. From the perspective of backbone architectures, existing methods can be broadly categorized into RNN-, CNN-, MLP-, and Transformer-based methods. 
Typically, RNN-based models~\cite{SegRNN, hewamalage2021recurrent, RNN1} leverage recurrent structures to model temporal state transitions, while CNN-based models~\cite{Timesnet, SCINet, luo2024moderntcn} employ temporal convolutional networks (TCNs) to extract local variation patterns. However, both RNN- and CNN-based methods are constrained by limited effective receptive fields, which hinders their ability to capture long-term dependencies. 
Lightweight MLP-based models~\cite{DLinear, timemixer} utilize stacked fully connected layers to model temporal patterns, where dense connections implicitly capture measurement-free relationships among time points within each variable~\cite{liuitransformer}. These linear forecasters are primarily designed for forecasting tasks and are computationally efficient. Nevertheless, due to their architectural simplicity and limited representational capacity, MLP-based models suffer from an information bottleneck, making it difficult to represent long-range and complex temporal dependencies. 
Transformer-based models, which adopt attention mechanisms to capture pairwise dependencies with a global receptive field, have become mainstream for MTS modeling~\cite{Informer, fedformer, wu2021autoformer, Pyraformer, PatchTST, Crossformer, liuitransformer}. 
These methods typically employ different tokenization strategies~\cite{jia2025principles, qu2025tokenrec, qu2025diffusion, zhou2025hd}, such as point-wise tokens~\cite{Informer, fedformer, wu2021autoformer, Pyraformer}, patch-wise tokens~\cite{PatchTST, Crossformer}, and series-wise tokens~\cite{liuitransformer}, and then apply attention to model temporal dependencies, variate dependencies, or both, as illustrated in Figure~\ref{fig:motivation}(a)(i-iv).
Despite their remarkable performance, Transformer-based models face growing concerns about their \textbf{\textit{quadratic computational complexity}}~\cite{TSMixer, DLinear}, leading to substantial computational demands and memory overhead. Specifically, the cost of self-attention scales quadratically with the token length, as listed in Figure~\ref{fig:motivation}(b)(i–iv). For an MTS instance with $N$ variates, each represented by $L$ tokens embedded in $D$ dimensions, the complexity of fully self-attention models can reach unacceptably $\mathcal{O}(N^2L^2D)$~\cite{liutimer, woounified}. As both $N$ and $L$ grow, quadratic scaling becomes a major bottleneck for long-term, large-scale MTS modeling, motivating the development of more computationally efficient architectures for MTS analysis.

% Mamba
Recently, \textbf{Mamba} has emerged as a promising backbone for capturing complex dependencies in sequential data~\cite{gu2023mamba,dao2024transformers,qu2024survey}. By leveraging selective mechanisms and hardware-aware computational designs~\cite{qu2024survey}, Mamba attains modeling performance comparable to Transformers while maintaining linear complexity with respect to sequence length. Furthermore, Mamba-2~\cite{dao2024transformers} further improves efficiency by exploiting the structured State Space Duality (SSD) property to reformulate computations into highly parallelizable matrix operations. With \textit{high expressiveness} enabled by the selective mechanism, \textit{efficient training and inference} supported by parallel computation, and \textit{linear scalability} with respect to context length, Mamba offers a compelling alternative to attention-based architectures for sequential modeling. Consequently, an increasing number of Mamba-based models have been proposed across various domains, such as Jamba~\cite{lieber2024jamba} in natural language processing, Vision Mamba~\cite{zhu2024vision} in computer vision, Caduceus~\cite{schiff2024caduceus} in genomics, and SSD4Rec~\cite{qu2024ssd4rec} for recommender systems.

Despite the recent success of Mamba-style techniques in enabling efficient MTS analysis, directly applying standard Mamba to multivariate time series often yields suboptimal performance compared with state-of-the-art methods~\cite{wang2024deep}. This gap can be mainly attributed to the following limitations.
(1) \textbf{Difficulty in disentangling entangled temporal and variate contexts:}
Unlike language tokens, which typically lie in relatively discrete contexts, MTS observations are jointly governed by intra-series temporal dynamics and inter-series interactions. These factors often overlap and intertwine, making it difficult for vanilla Mamba to disentangle meaningful and structured representation.
(2) \textbf{Lack of explicit cross-variate modeling:}
Mamba was originally developed for language-like, essentially univariate sequences, and therefore lacks explicit mechanisms for capturing dependencies among multiple correlated variates. Prior studies have shown that modeling inter-variable interactions is crucial for effective MTS representation learning~\cite{liuitransformer, Crossformer, LIFT}.
(3) \textbf{Insufficient explicit modeling of lag dependencies:}
Mamba relies on recurrent hidden-state updates, in which each state depends primarily on the immediately preceding timestep. This design may limit its ability to explicitly capture lagged effects that are pervasive in real-world MTS, where changes in one variate may influence another only after a non-negligible delay~\cite{long2024unveiling, LIFT, jiang2023pdformer,an2025damba}. For example, in traffic forecasting, an accident in one region may not immediately affect adjacent areas; instead, its impact often propagates through the network over several minutes. Precisely capturing such propagation delays is therefore crucial for
accurate dependency modeling and reliable prediction.

Motivated by the above limitations and the urgent practical demands of long-term, large-scale MTS analysis, we aim to develop a Mamba-based time-series representation backbone that simultaneously captures intricate temporal dependencies and delay-aware cross-variate interactions while preserving linear-time efficiency. Toward this goal, we explicitly decompose the time-series context into (i) \textit{intra-series temporal dynamics}, which can be learned independently and in parallel for each variate, and (ii) \textit{inter-series interactions}, which are modeled in a delay-aware manner to capture cross-variate dependencies, as illustrated in Figure~\ref{fig:motivation}(a)(v).

To this end, we propose a dual-path \textbf{De}lay-aware \textbf{Ma}mba for efficient multivariate time series analysis, namely \textbf{\ourname{}}.
Specifically, \ourname{} first adaptively selects task-relevant spectra via an Adaptive Fourier Filter and decomposes the input into a \textit{Cross-Time Component} and a \textit{Cross-Variate Component}.
These two components are then fed into stacked \textbf{DuoMNet} blocks.
Each DuoMNet block contains two parallel paths, each coupled with a scan operator and a lightweight Mamba-based module.
Concretely, the \textit{Cross-Time Scan} serializes each variate along the temporal axis into a 1D sequence, which is processed by \textit{Mamba-SSD}.
This design supports MTS-independent parallel computation across variates and efficiently captures long-range temporal dependencies within each series, with computation scaling linearly with the series length.
In parallel, the \textit{Cross-Variate Scan} reorganizes tokens to emphasize inter-variable interactions at each token step.
The scanned sequence is processed by \textit{Mamba-DALA}, which integrates Delay-Aware Linear Attention (DALA) to explicitly model cross-variate dependencies with both a global correlation delay and a token-level relative delay.
As a result, \ourname{} achieves delay-aware variate interaction modeling while preserving linear-time computation.
Finally, we fuse the temporal-path and variate-path representations through a weighted fusion layer and project to task-specific outputs via lightweight heads.
To sum up, our major contributions include:
\begin{itemize}[leftmargin=*]
    \item We propose \textbf{\ourname{}}, a dual-path Mamba-based backbone for general MTS analysis that jointly captures intra-series temporal dynamics via a temporal path and delay-aware cross-variate dependencies via a variate path. Benefiting from a linear-time state-space backbone, \ourname{} achieves a favorable trade-off between accuracy and efficiency for long-term and large-scale MTS modeling.
    
    \item We introduce Mamba-SSD to model intra-series temporal dependencies within each variate. It accommodates variable-length inputs and enables series-independent, parallel modeling via block-based matrix multiplication, thereby improving scalability and throughput.
    
    \item We design \DALA{}, which integrates Delay-Aware Linear Attention to explicitly model cross-variate interactions with both global correlation delays and token-level relative delays, thereby enhancing fine-grained, delay-sensitive dependency learning.
    
    \item We conduct extensive experiments on five mainstream tasks, including long- and short-term forecasting, imputation, classification, and anomaly detection. The results demonstrate that \ourname{} achieves consistently strong performance across tasks while significantly reducing training time and GPU memory usage, suggesting its practicality and scalability for real-world large-scale MTS analysis.
\end{itemize}

The remainder of this paper is structured as follows. Section~\ref{sec:pre} introduces the preliminaries.
Section~\ref{sec:method} introduces the proposed model, which is evaluated and discussed in Section~\ref{sec:evaluation}.
Then, Section~\ref{sec:literature} summarizes the recent development of time series analysis. Finally, conclusions are drawn in Section~\ref{sec:conclusion}.

\section{Preliminaries}
\label{sec:pre}
\subsection{\textbf{Notations and Definitions}}

We denote a multivariate time series (MTS) within a lookback window as
$\mathcal{X}\in\mathbb{R}^{N\times T}$, where $N$ denotes the number of variates and
$T$ denotes the number of time steps. The $i$-th variate is represented as
$\mathcal{X}_{i,:}=\{x_{i1},x_{i2},\ldots,x_{iT}\}\in\mathbb{R}^{T}$, while the multivariate observation at time step $t$ is denoted by
$\mathcal{X}_{:,t}=\{x_{1t},x_{2t},\ldots,x_{Nt}\}\in\mathbb{R}^{N}$.

The primary objective of this work is to learn a task-agnostic representation
$\mathbf{Z}=\mathcal{F}_{\Theta}(\mathcal{X})$, which can be adapted to diverse downstream tasks through task-specific heads.
These tasks include point-level tasks, such as forecasting the future $H$ time steps with output
$\mathcal{Y}\in\mathbb{R}^{N\times H}$, as well as anomaly detection and missing-value imputation on the input window with output
$\mathcal{Y}\in\mathbb{R}^{N\times T}$.
In addition, we consider the sequence-level classification task, which assigns the input MTS to one of $C$ categories, with output
$\mathcal{Y}\in\mathbb{R}^{1\times C}$.

\subsection{\textbf{Mamba}}
\subsubsection{\textbf{State Space Model (SSM)}}The classical State Space Models (SSMs) describes the state evolution of a linear time-invariant system, which map input signal $\mathbf{x}=\{x(1),\cdots,x(t),\cdots\} \in \mathbb{R}^{L\times D} \mapsto \mathbf{y}=\{y(1),\cdots,y(t),\cdots\} \in \mathbb{R}^{L \times D}$ through implicit latent state $h(t) \in \mathbb{R}^{N\times D}$, where $t$, $L$, $D$ and $N$ indicate the time step, sequence length, channel number of the signal and state size, respectively.
These models can be formulated as the following linear ordinary differential equations:
\begin{equation}
\label{eq:ode}
\begin{split}
h^{\prime}(t)=\mathbf{A}h(t)+\mathbf{B}x(t),\quad
y(t)=\mathbf{C}h(t)+\mathbf{D}x(t),
\end{split}
\end{equation}
where $\mathbf{A} \in \mathbb{R}^{N\times N} $ is the state transition matrix that describes how states change over time, $\mathbf{B} \in \mathbb{R}^{N\times D} $ is the input matrix that controls how inputs affect state changes, $\mathbf{C} \in \mathbb{R}^{N\times D} $ denotes the output matrix that indicates how outputs are generated based on current states and $\mathbf{D} \in \mathbb{R}^{D} $ represents the command coefficient that determines how inputs affect outputs directly. Most SSMs exclude the second term in the observation equation, i.e., set $\mathbf{D}x(t)=0$, which corresponds to a skip connection in deep learning models.
The time-continuous nature poses challenges for integration into deep learning architectures.
To alleviate this issue, most methods utilize the Zero-Order Hold rule~\cite{gu2023mamba} to discretize continuous time into $K$ intervals, which assumes that the function value remains constant over the interval $\boldsymbol{\Delta} \in \mathbb{R}^D$. The Eq.~\eqref{eq:ode} can be reformulated as:
\begin{equation}
\label{eq:Discrete}
\begin{split}
    h_t = \overline{\mathbf{A}}  h_{t-1} + \overline{\mathbf{B}}  x_t,
    \quad y_t = \mathbf{C}  h_t, 
\end{split}
\end{equation}
where $\overline{\mathbf{A}} = \exp\left( \boldsymbol{\Delta} \mathbf{A} \right)$ and  $\overline{\mathbf{B}}=(\boldsymbol{\Delta} \mathbf{A})^{-1}(\exp (\boldsymbol{\Delta} \mathbf{A})-\mathbf{I}) \cdot \boldsymbol{\Delta} \mathbf{B}$.
Discrete SSMs can be interpreted as a combination of CNNs and RNNs. Typically, the model employs a convolutional mode for efficient, parallelizable training and switches to a recurrent mode for efficient autoregressive inference. The formulations in Eq.~\eqref{eq:Discrete} are equivalent to the following convolution~\cite{gu2020hippo}:
\begin{equation}
\begin{split}
&\overline{\mathbf{K}} = (\mathbf{C} \overline{\mathbf{B}},\mathbf{C} \overline{\mathbf{A}}\overline{\mathbf{B}},\cdots,\mathbf{C} \overline{\mathbf{A}}^k \overline{\mathbf{B}}, \cdots),\\
&\mathbf{y} = \mathbf{x} * \overline{\mathbf{K}},    
\end{split}
\end{equation}
Thus, the overall process can be represented as:
\begin{equation}
\mathbf{y} = \textbf{SSM}(\overline{\mathbf{A}}, \overline{\mathbf{B}}, \mathbf{C})(\mathbf{x}).
\end{equation}

\subsubsection{\textbf{Selective State Space Model (Mamba)}}
The discrete SSMs are based on data-independent parameters, meaning that parameters $\bar{\mathbf{A}}$, $\bar{\mathbf{B}}$, and ${\mathbf{C}}$ are time-invariant and the same for any input, limiting their effectiveness in compressing context into a smaller state~\cite{gu2023mamba}. 
Mamba introduces a selective mechanism to selectively retain, propagate, or suppress information according to the current token~\cite{gu2023mamba, qu2024survey}. This directly addresses a major weakness of earlier subquadratic sequence models, namely their limited ability for content-based reasoning. As a result, Mamba preserves a form of context-aware sequence modeling that is much closer in spirit to attention, while avoiding the dense token-to-token interactions of standard self-attention.
Specifically, it utilizes Linear Projection to parameterize the weight matrices $\{\mathbf{B}_{t}\}_{t=1}^{L} $, $\{\mathbf{C}_{t}\}_{t=1}^{L} $ and $\{\mathbf{\Delta}_{t}\}_{t=1}^{L}$ according to model input $\{x_t\}_{t=1}^{L}$, improving the context-aware ability, i.e.,:
\begin{equation}
\begin{split}
    &\overline{\mathbf{B}}_t = \texttt{Linear}_{\mathbf{B}}(x_t),\\
    &\overline{\mathbf{C}}_t = \texttt{Linear}_{\mathbf{C}}(x_t), \\
    &\boldsymbol{\Delta}_t = \texttt{Softplus}\left(\texttt{Linear}_{\boldsymbol{\Delta}}(x_t)\right).
\end{split}
\end{equation}
Then the output sequence $\{y_{t}\}_{i=t}^{L} $ can be computed with those input-adaptive discretized parameters as follows:
\begin{equation}
\label{eq:ssm}
h_{t}=\overline{\mathbf{A}}_{t} h_{t-1}+\overline{\mathbf{B}}_{t} x_{t}, \quad y_{t}=\overline{\mathbf{C}_{t}} h_{t}.
\end{equation}

During training, both computation and memory scale linearly with sequence length, rather than quadratically as in standard Transformers. During autoregressive inference, Mamba only needs to update a recurrent hidden state rather than maintaining an ever-growing KV cache over the full context. Moreover, Mamba introduces a hardware-aware parallel scan algorithm~\cite{harris2007parallel}, enabling selective and time-varying SSMs to be practically efficient on modern accelerators.

\subsubsection{\textbf{Mamba-2}}
Mamba-2~\cite{dao2024transformers} introduces a comprehensive framework, Structured State-Space Duality (SSD). Leveraging SSD, it reformulates SSMs as semi-separable matrices via matrix transformation and develops a more hardware-efficient computation method based on block-decomposed matrix multiplication, i.e.,
\begin{equation}
\begin{split}
\mathbf{y} &= \textbf{SSD}(\mathbf{A}, \mathbf{B}, \mathbf{C})(\mathbf{x}) \\
&= \mathbf{M}\mathbf{x}\\
\text{s.t.}~ \mathbf{M} &=
\begin{pmatrix}
\mathbf{C}_0^\top \mathbf{A}_{0:0} \mathbf{B}_0 & & & &\\
\mathbf{C}_1^\top \mathbf{A}_{1:0} \mathbf{B}_0 & \mathbf{C}_1^\top \mathbf{A}_{1:1} \mathbf{B}_1 & & & \\
\ldots & \ldots & \ldots & \ldots &  \\
\mathbf{C}_j^\top \mathbf{A}_{j:0} \mathbf{B}_0 & \mathbf{C}_j^\top \mathbf{A}_{j:1} \mathbf{B}_1 & \ldots & \mathbf{C}_j^\top \mathbf{A}_{j:i} \mathbf{B}_i  \\
\end{pmatrix}
\end{split}
\end{equation}
where $\mathbf{M}$ denotes the matrix form of SSMs that uses the sequentially semiseparable representation, and $\mathbf{M}_{ji} = \mathbf{C}_j^\top  \mathbf{A}_{j:i} \mathbf{B}_i$, $\mathbf{C}_j$ and $\mathbf{B}_i$ represent the selective space state matrices associated with input tokens $x_j$ and $x_i$, respectively. $\mathbf{A}_{j:i}$ denotes the selective matrix of hidden states corresponding to the input tokens ranging from $j$ to $i$. 
Mamba-2 achieves a 2-8$\times$ faster training process than Mamba-1's parallel associative scan while remaining competitive with Transformers. 
\section{The Proposed Method: \ourname{}}
\label{sec:method}

\subsection{\textbf{Structure Overview}} 
As shown in Figure~\ref{workflow} (left), \ourname{} is designed in a challenge-driven manner to address three key limitations in multivariate time-series (MTS) modeling: the entanglement of temporal and variate contexts, the lack of explicit cross-variate modeling, and the inadequate modeling of lag-aware interactions. Accordingly, the framework consists of three stages: the \emph{Adaptive Fourier Filter}, stacked \emph{DuoMNet Blocks}, and \emph{Task-specific Projection}.
The \emph{Adaptive Fourier Filter} first decomposes the raw MTS into an intra-series temporal component and an inter-series variate component, thereby reducing interference between temporal dynamics and cross-variate interactions before representation learning.
Built upon this decomposition, the stacked \emph{DuoMNet Blocks} address the second challenge through a dual-path design. As shown in Figure~\ref{workflow} (right), each block contains a temporal path for modeling intra-series dependencies and a variate path for capturing inter-series dependencies, allowing the model to explicitly and separately learn these two types of structure.
Furthermore, the variate path incorporates delay-aware interaction modeling, enabling \ourname{} to capture not only whether variables are correlated, but also how their interactions evolve.
Finally, the learned shared representation is passed to lightweight \emph{Task-specific Projection} layers for downstream tasks such as forecasting, imputation, classification, and anomaly detection.

\begin{figure*}[h]
\includegraphics[width=0.98\textwidth]{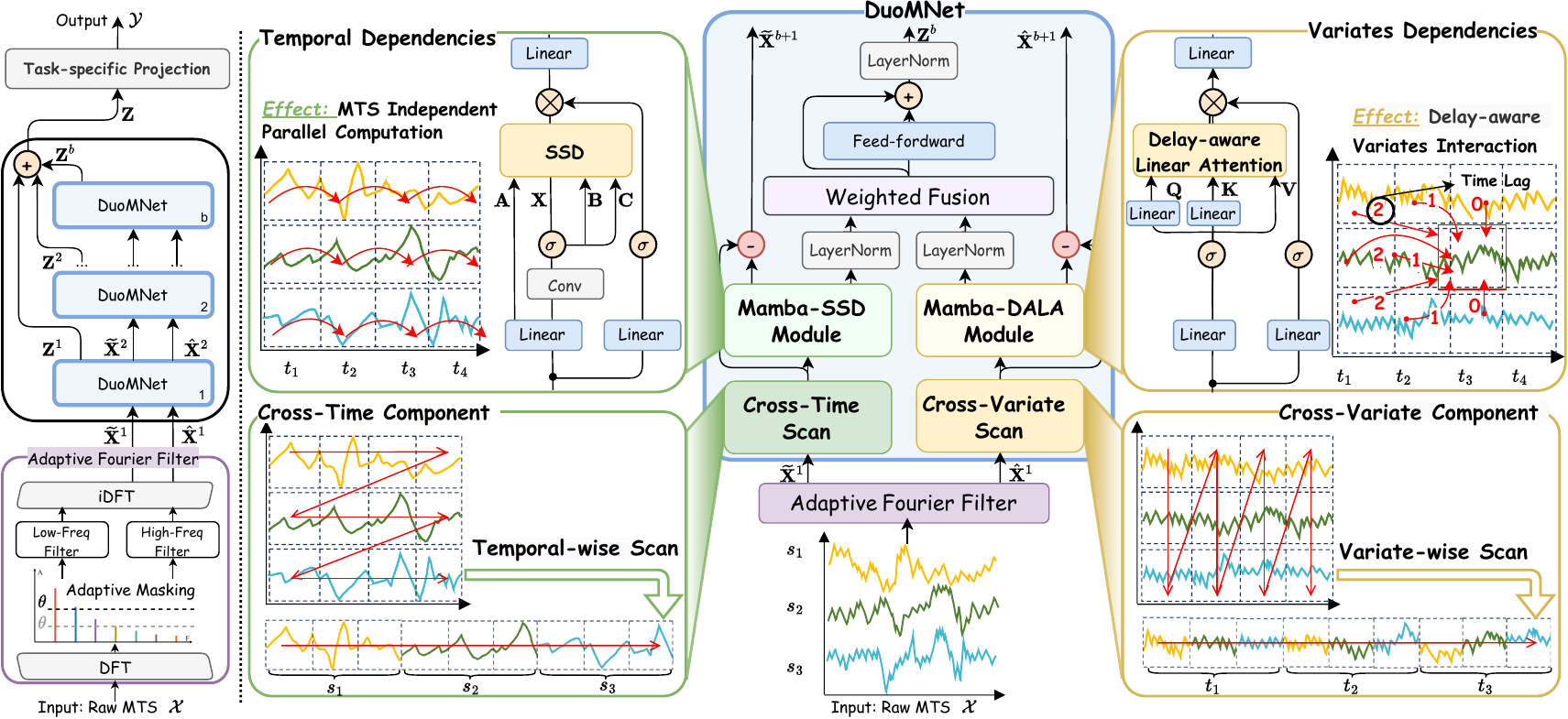}
\caption{The overall framework of \ourname{}. The proposed \ourname{} (Left) comprises three key components: an Adaptive Fourier Filter, stacked DuoMNet Blocks, and Task-specific Projection. Importantly, the DuoMNet Block (Right) integrates insights from the \SSD{} and \DALA{} paths to comprehensively model cross-time and cross-variate dependencies.}
\label{workflow}
\end{figure*} 

\subsubsection{\textbf{Adaptive Fourier Filter (AFF)}}

Real-world multivariate time series often exhibit overlapping and entangled patterns arising from (i) \textit{intra-series temporal dynamics} within each variate and (ii) \textit{inter-series interactions} among correlated variates. In the frequency domain, slowly varying structures such as trends and seasonal cycles are typically concentrated in low-frequency bands. In contrast, short-term variations, including abrupt changes and time-varying cross-variate effects, tend to appear at higher frequencies~\cite{yi2023frequency,liu2023koopa}. Meanwhile, different tasks emphasize different spectral contents: numerical prediction tasks (such as forecasting, imputation, and anomaly detection) are usually sensitive to local dynamics, whereas semantic tasks (e.g., classification) rely more on global and long-range patterns~\cite{yue2022ts2vec}. These characteristics motivate a task-adaptive spectral decomposition that separates and highlights task-relevant components.

To facilitate modeling of intra-series dynamics and inter-series interactions, we split the input into two complementary frequency subsets. The \emph{globally dominant} frequencies capture broadly shared, long-term structures and mainly reflect per-variate temporal dynamics, while the \emph{residual} frequencies encode more window-specific variations that are often indicative of time-varying cross-variate interactions. Concretely, given an input window $\mathcal{X}\in\mathbb{R}^{N\times T}$ with length $T$, we first apply the Fast Fourier Transform (FFT) along the temporal axis for each variate and obtain the averaged amplitude for each frequency index in $\mathcal{S}=\{0,1,\dots,\lfloor T/2\rfloor\}$. We rank $\mathcal{S}$ by amplitude, select the top $\theta$ fraction as $\mathcal{S}_\theta$, and denote the complement by $\overline{\mathcal{S}_\theta}=\mathcal{S}\setminus\mathcal{S}_\theta$. The Adaptive Fourier Filter $\operatorname{AFF}(\cdot)$ then decomposes $\mathcal{X}$ as
\begin{equation}
\begin{split}
\text{Cross-Time Component}:\widetilde{\mathcal{X}} &= \operatorname{FFT}^{-1}\!\Big(\operatorname{Filter}\big(\mathcal{S}_\theta,\operatorname{FFT}(\mathcal{X})\big)\Big),\\
\text{Cross-Variate Component}:\hat{\mathcal{X}} &= \operatorname{FFT}^{-1}\!\Big(\operatorname{Filter}\big(\overline{\mathcal{S}_\theta},\operatorname{FFT}(\mathcal{X})\big)\Big),
\end{split}
\label{eq:aff}
\end{equation}
where $\operatorname{FFT}^{-1}$ denotes the inverse FFT and $\operatorname{Filter}(\cdot)$ retains only the coefficients indexed by the specified set. The \textit{Cross-Time Component} $\widetilde{\mathcal{X}}$ is dominated by global, low-frequency spectral components that primarily reflect intra-series temporal dynamics, and \textit{Cross-Variate Component} $\hat{\mathcal{X}}$ preserves the remaining variations, typically higher-frequency and window-specific, which are related to inter-series interactions.

This decomposition is particularly valuable when a state-space model serves as the backbone for MTS modeling. In Mamba, the HiPPO-based state-space formulation has been shown to induce approximately orthogonal transformations~\cite{gu2022train}. However, the orthogonality of the learned transform does not imply that the mixed components in the raw MTS can be cleanly separated. If two underlying components are \emph{not} orthogonal in the input space, their projections onto any orthogonal basis must share some basis directions. We formalize this limitation below.

\begin{theorem}[Non-orthogonality implies basis overlap]
\label{thm:nonorth_disentangle}
Given the input series $\mathcal{X}=\widetilde{\mathcal{X}}+\hat{\mathcal{X}}$ and let $E=\{e_1,\dots,e_N\}$ be an orthogonal basis. Suppose
\[
\widetilde{\mathcal{X}}=\sum_{i=1}^N a_i e_i,\qquad
\hat{\mathcal{X}}=\sum_{i=1}^N b_i e_i,
\]
and define the supports $\widetilde{E}=\{e_i\in E:\ a_i\neq 0\}$ and $\hat{E}=\{e_i\in E:\ b_i\neq 0\}$.
If $\widetilde{\mathcal{X}}$ and $\hat{\mathcal{X}}$ are not orthogonal, i.e., $\langle \widetilde{\mathcal{X}}, \hat{\mathcal{X}}\rangle \neq 0$, then $\widetilde{E}\cap \hat{E}\neq \emptyset$.
Equivalently, two non-orthogonal components cannot be represented on disjoint subsets of an orthogonal basis.
\end{theorem}

\begin{proof}
Since $E=\{e_1,\dots,e_N\}$ is an orthogonal basis, we have $\langle e_i,e_j\rangle = 0$ for $i\neq j$ and
$\langle e_i,e_i\rangle>0$ for all $i$. Thus,
\begin{equation}
\begin{split}
\langle \widetilde{\mathcal{X}}, \hat{\mathcal{X}}\rangle
&= \left\langle \sum_{i=1}^N a_i e_i,\ \sum_{j=1}^N b_j e_j \right\rangle \\
&= \sum_{i=1}^N \sum_{j=1}^N a_i b_j \langle e_i, e_j\rangle \\
&= \sum_{i=1}^N a_i b_i \langle e_i, e_i\rangle,
\end{split}
\end{equation}
If $\widetilde{E}\cap \hat{E}=\emptyset$, then for every $i$ at least one of $a_i$ or $b_i$ is zero, implying $a_i b_i=0$ for all $i$, which implies
\begin{equation}
\langle \widetilde{\mathcal{X}}, \hat{\mathcal{X}}\rangle
= \sum_{i=1}^N a_i b_i \langle e_i,e_i\rangle = 0.
\end{equation}
This contradicts the assumption $\langle \widetilde{\mathcal{X}}, \hat{\mathcal{X}}\rangle \neq 0$.
Therefore, there exists at least one index $i$ such that $a_i\neq 0$ and $b_i\neq 0$, i.e., $e_i\in \widetilde{E}\cap \hat{E}$.
\end{proof}

Theorem~\ref{thm:nonorth_disentangle} implies that if the Cross-Time component $\widetilde{\mathcal{X}}$ and the Cross-Variate component $\hat{\mathcal{X}}$ are non-orthogonal, their representations on any orthogonal basis must overlap, making a complete separation by an orthogonal transform impossible. In practice, this non-orthogonality often arises because different generative factors are \emph{not} confined to disjoint frequency bands: for example, trend, seasonality, and interaction-driven fluctuations can partially overlap in the same spectral range and thus superimpose in both the time and frequency domains. Therefore, although HiPPO-based SSMs tend to learn orthogonal transformations~\cite{gu2022train}, applying a single Mamba model directly to raw MTS may still entangle these factors. In contrast, our proposed $\operatorname{AFF}(\cdot)$ explicitly decomposes $\mathcal{X}$ into complementary spectral parts, allowing subsequent modules to model them separately and thereby reducing representational entanglement and avoiding redundant modeling capacity.

\subsection{\textbf{DuoMNet Block}}
In this section, we detail the design of the proposed DuoMNet block, which comprises a temporal path to capture cross-time dependencies and a variate path to learn cross-variate dependencies.

\subsubsection{\textbf{Cross-Time Scan and Cross-Variate Scan}}
Given the decomposed Cross-Time Component $\widetilde{\mathcal{X}}\in\mathbb{R}^{N\times T}$ and Cross-Variate Component $\hat{\mathcal{X}}\in\mathbb{R}^{N\times T}$, we tokenize each component into patch
tokens to reduce point-wise redundancy and encode local temporal semantics~\cite{PatchTST, Crossformer}.
Concretely, for each variate $i$, we apply reversible instance normalization (RevIN)~\cite{kim2021reversible} and
partition the normalized series into $L$ contiguous patches of length $P$. A lightweight 1-D convolutional encoder
maps each patch to a $D$-dimensional embedding. To enable complementary dependency modeling, we organize the resulting
embeddings with two scan orders:

(a) \textbf{\textit{Cross-Time Scan (temporal-wise).}}
We preserve the temporal order of patch tokens in the Cross-Time Component and stack all variates in parallel:
\begin{equation}
\widetilde{\mathbf{X}}
=\big[\widetilde{\mathbf{X}}_{1,:};\widetilde{\mathbf{X}}_{2,:};\ldots;\widetilde{\mathbf{X}}_{N,:}\big]
\in\mathbb{R}^{N\times L\times D},
\end{equation}
where $\widetilde{\mathbf{X}}_{i,:}\in\mathbb{R}^{L\times D}$ contains the $L$ patch embeddings of the $i$-th variate.
This scan order enables the temporal module to model \emph{intra-series} dependencies along the temporal axis independently for each variate.

(b) \textbf{\textit{Cross-Variate Scan (variate-wise).}}
We instead group tokens in the Cross-Variate Component by the same patch position and scan across variates:
\begin{equation}
\hat{\mathbf{X}}
=\big(\hat{\mathbf{X}}_{:,1},\hat{\mathbf{X}}_{:,2},\ldots,\hat{\mathbf{X}}_{:,L}\big)
\in\mathbb{R}^{L\times N\times D},
\end{equation}
where $\hat{\mathbf{X}}_{:,l}=\big(\hat{\mathbf{x}}_{1,l},\ldots,\hat{\mathbf{x}}_{N,l}\big)\in\mathbb{R}^{N\times D}$
collects the embeddings of all $N$ variates at the $l$-th token step. This scan order preserves the chronological order over token indices and facilitates modeling
time-varying cross-variate interactions at different token steps.

These two scanned representations, $\widetilde{\mathbf{X}}$ and $\hat{\mathbf{X}}$, are then fed into subsequent modules
to learn cross-time and cross-variate dependencies, respectively.

\subsubsection{\textbf{Cross-Time Dependencies Modeling via Mamba-SSD}}

Given the temporal scanned embeddings $\widetilde{\mathbf{X}}\in\mathbb{R}^{N\times L\times D}$ of the Cross-Time component, our goal is to capture \emph{intra-series} temporal dependencies within each variate.
We therefore adopt an \textbf{\textit{independent and parallel}} strategy: each variate is modeled without cross-variate coupling, enabling efficient parallel computation across $N$ variates. The proposed Mamba-SSD module consists of four stages: (a) two-branch gating and local mixing, (b) selective SSM parameterization, (c) Structured State Space Duality (SSD) based parallel scan, and (d) gated output projection.

\textit{(a) \textbf{Two-branch Gating and Local Mixing.}}
We first apply linear projections to obtain a content branch $\widetilde{\mathbf{X}}\in\mathbb{R}^{N\times L\times D_u}$ and a gate branch $\widetilde{\mathbf{X}}_{\text{gate}}\in\mathbb{R}^{N\times L\times D_u}$. The gate branch serves as a residual controller that adaptively filters and reweights the content features.
We then perform local temporal mixing by applying a lightweight 1-D convolution along the token axis on the content branch (for each variate in parallel), which injects an explicit short-range inductive bias to model local continuity that is common in real-world time series, and complements the long-range selective SSM by providing locally aggregated features that stabilize token-dependent parameter generation.

\textit{(b) \textbf{Selective SSM parameterization.}}
Following Mamba-style selective SSMs, we generate token-dependent parameters from $\widetilde{\mathbf{X}}$.
Let $D_h$ denote the state dimension. The step size and the output projections are computed as
\begin{equation}
\label{eq:mssd-params-batch}
\begin{split}
    &\Delta=\operatorname{softplus}(\widetilde{\mathbf{X}}\mathbf{W}_{\Delta}+\mathbf{b}_{\Delta})\in\mathbb{R}^{N\times L\times D_h},\\
&\mathbf{B}=\widetilde{\mathbf{X}}\mathbf{W}_{B}+\mathbf{b}_{B}\in\mathbb{R}^{N\times L\times D_h},\\
&\mathbf{C}=\widetilde{\mathbf{X}}\mathbf{W}_{C}+\mathbf{b}_{C}\in\mathbb{R}^{N\times L\times D_h},
\end{split}
\end{equation}
where $\mathbf{W}_{\Delta},\mathbf{W}_{B},\mathbf{W}_{C}\in\mathbb{R}^{D_u\times D_h}$.
To ensure stable dynamics, we parameterize the continuous-time transition with a negative diagonal vector
$\mathbf{A}=-\exp(\mathbf{A}_{\log})\in\mathbb{R}^{D_h}$ and discretize it using $\Delta$:
\begin{equation}
\label{eq:mssd-discretize-batch}
\begin{split}
&\overline{\mathbf{A}}=\exp(\Delta\odot \mathbf{A})\in\mathbb{R}^{N\times L\times D_h},\\
&\overline{\mathbf{B}}=(1-\bar{\mathbf{A}})\odot \mathbf{B}\in\mathbb{R}^{N\times L\times D_h},
\end{split}
\end{equation}
where $\odot$ denotes element-wise multiplication. Here, $\Delta$ controls the effective memory scale at each token, while
$\mathbf{B}$ and $\mathbf{C}$ modulate how content is written into and read from the latent state, yielding content-adaptive dynamics.

\textit{(c) \textbf{SSD-based Parallel Computation.}} Given the input token sequence $\widetilde{\mathbf{X}} \in \mathbb{R}^{N\times L \times D_u}$, the selective SSM at step $t$ can be formulated as follows:
\begin{equation}
\begin{split}
h_t &= \overline{\mathbf{A}}_t h_{t-1} + \overline{\mathbf{B}}_t \widetilde{\mathbf{X}}_t,  \\
\mathbf{Y}_t &= \mathbf{C}_t^\top h_t,
\label{eq:SSM-2}
\end{split}
\end{equation}
which yields $\mathbf{Y}\in\mathbb{R}^{N\times L\times D_h}$.
Equivalently, the induced linear operator is \emph{block-diagonal} across variates by leveraging Structured State-Space Duality (SSD) properties in Mamba-2~\cite{dao2024transformers}. Specifically, the selective SSM model presented in Eq.~(\ref{eq:SSM-2}) can be reformulated as:
\begin{equation}
\begin{split}
\mathbf{Y}_{\text{time}} &= \operatorname{SSD}(\mathbf{A}, \mathbf{B}, \mathbf{C})(\widetilde{\mathbf{X}})=\mathbf{M}\widetilde{\mathbf{X}} \\
\text{s.t.}~\mathbf{M} &=
\begin{pmatrix}
\mathbf{C}_0^\top \mathbf{A}_{0:0} \mathbf{B}_0 & & & &\\
\mathbf{C}_1^\top \mathbf{A}_{1:0} \mathbf{B}_0 & \mathbf{C}_1^\top \mathbf{A}_{1:1} \mathbf{B}_1 & & & \\
\ldots & \ldots & \ldots & \ldots &  \\
\mathbf{C}_j^\top \mathbf{A}_{j:0} \mathbf{B}_0 & \mathbf{C}_j^\top \mathbf{A}_{j:1} \mathbf{B}_1 & \ldots & \mathbf{C}_j^\top \mathbf{A}_{j:i} \mathbf{B}_i  \\
\end{pmatrix} ,
\end{split}
\end{equation}
where $\mathbf{M}$ denotes the matrix form of SSMs that uses the sequentially semiseparable representation, and $\mathbf{M}_{ji} = \mathbf{C}_j^\top  \mathbf{A}_{j:i} \mathbf{B}_i$
represents the mapping from the $i$-th input token to the $j$-th output token. Here, $\mathbf{B}_i$ and $\mathbf{C}_j$ are the selective state-space parameters associated with the $i$-th and $j$-th tokens, respectively, and $\mathbf{A}_{j:i}$ denotes the product of the transition terms from step $i$ to $j$.

By representing state-space models as semiseparable matrices via matrix transformations, the structured state-space model enables block-based matrix multiplication, allowing content-aware modeling analogous to attention while achieving subquadratic-time computation.
Building on this insight, we perform the SSD operator in parallel over all $N$ variates:
\begin{equation}
\label{eq:mssd-blockdiag}
\begin{split}
\mathbf{Y}_{\text{time}}&=\operatorname{SSD}(\mathbf{A},\mathbf{B},\mathbf{C})(\widetilde{\mathbf{X}})\\
&=\big(\mathbf{M}_1\widetilde{\mathbf{X}}_{1,:};\ \mathbf{M}_2\widetilde{\mathbf{X}}_{2,:};\ \ldots;\ \mathbf{M}_N\widetilde{\mathbf{X}}_{N,:}\big),
\end{split}
\end{equation}
where each $\mathbf{M}_n$ is a semiseparable operator acting only on the $n$-th variate sequence. Notably, during this process, we prevent any cross-series information flow by re-initializing (i.e., zeroing) the hidden state at the beginning of each series, thereby forcing the Mamba-SSD path to focus exclusively on temporal modeling. This explicitly enforces that the Cross-Time branch captures temporal dynamics \emph{within} each variate while remaining fully parallelizable across $N$ variates.

\textit{(d) \textbf{Gated Output Projection.}} Finally, we combine the SSD output $\mathbf{Y}_{\text{time}}$ from the content branch with the gate branch to obtain the final output:
\begin{equation}
    \widetilde{\mathbf{Y}} =\operatorname{Linear} (\mathbf{Y}_{\text{time}} \odot \sigma(\widetilde{\mathbf{X}}_{\text{gate}})),
\end{equation}
where $\sigma(\cdot)$ is the sigmoid function and $\odot$ denotes element-wise multiplication. The purpose of this gated projection is to provide an explicit, token-adaptive modulation on the propagated state features, selectively enhancing informative temporal responses while attenuating irrelevant or noisy activations, thereby improving robustness and expressiveness without introducing cross-variational coupling. Thus, the overall process in the Mamba-SSD module is:
\begin{equation}
\label{eq:mssd-out-batch}
\widetilde{\mathbf{Y}}=\operatorname{Mamba\mbox{-}SSD}(\widetilde{\mathbf{X}}).
\end{equation}

\subsubsection{\textbf{Cross-Variate Dependencies Modeling via Mamba-DALA}}

A critical yet often overlooked factor in multivariate time series modeling is the \textit{delay dependency}: variations in one variate may influence another only after a \emph{time lag}. 
For instance, in traffic forecasting, congestion at an intersection typically propagates to nearby intersections with a non-negligible delay rather than instantaneously. 
This lagged propagation property makes forecasting
more challenging.
Accurately capturing such delayed interactions is essential for both predictive performance and practical utility, yet this effect is often overlooked in existing time series models. 

To efficiently capture cross-variate dependencies in the Cross-Variate component $\hat{\mathbf{X}}\in\mathbb{R}^{L\times N\times D}$ while explicitly accounting for delay effect, we develop a Mamba-style mixing block that replaces the recurrent SSM in the Mamba backbone with a \textbf{Delay-Aware Linear Attention (DALA)} operator, termed \textbf{Mamba-DALA}. 
Built upon Mamba’s efficient block design~\cite{han2024demystify}, Mamba-DALA performs token-wise pairwise interaction modeling across variates, yielding delay-aware cross-variate aggregation among multiple correlated latent series.

\textit{(a) \textbf{Mamba-style Gating and Parameterization.}}
Following the two-branch design in Mamba, we apply linear projections to obtain a content branch $\hat{\mathbf{X}}\in\mathbb{R}^{L\times N\times D_u}$ and a gate branch $\hat{\mathbf{X}}_{\text{gate}}\in\mathbb{R}^{L\times N\times D_u}$.
We further construct the linear-attention operator by
$\hat{\mathbf{X}}^{q}=\hat{\mathbf{X}}\mathbf{W}_Q,\quad
\hat{\mathbf{X}}^{k}=\hat{\mathbf{X}}\mathbf{W}_K,\quad
\hat{\mathbf{X}}^{v}=\hat{\mathbf{X}}\mathbf{W}_V$,
where $\mathbf{W}_Q,\mathbf{W}_K,\mathbf{W}_V\in\mathbb{R}^{D_u\times D_u}$.

\textit{(b) \textbf{Delay-Aware Linear Attention (DALA).}}
To robustly encode propagation delays, we consider two delay signals: a \emph{global correlation delay} and a \emph{token-level relative delay}. 
The key idea is to first correct coarse cross-series misalignment with a global offset, and then explicitly model the relative temporal distance between interacting tokens. In this way, the model can both identify the most plausible cross-series alignment and preserve ordered, causal interactions at the token level. Accordingly, we first estimate a correlation-based delay by maximizing the cross-correlation between variates, which serves as an explicit prior for cross-series alignment. Importantly, this correlation-based delay is not imposed as a hard alignment constraint. Instead, it provides a coarse global prior. At the same time, the final delay-aware interaction is further refined by token-level relative delay encoding via Rotary Position Embedding (RoPE)~\cite{su2024roformer}, which captures time lags between query-key token pairs.

\noindent\textbf{\textit{Global correlation delay.}} 
we first estimate the global propagation delay by maximizing the cross-correlation~\cite{azaria1984time} between latent correlated series $\mathcal{X}_{a,:}\in\mathbb{R}^{T}$ and $\mathcal{X}_{b,:}\in\mathbb{R}^{T}$ via temporal shifting~\cite{long2024unveiling}. Meanwhile, we quantify the \emph{correlation strength} as the peak correlation value:
\begin{equation}
\label{mcc}
\begin{split}
    \tau_{ab} &= \arg \max_t \textit{corr}(\mathcal{X}_{a,:} \overset{t}{\rightarrow} \mathcal{X}_{b,:}), \\
\rho_{ab}& =\textit{corr}(\mathcal{X}_{a,:} \overset{\tau_{ab}}{\rightarrow} \mathcal{X}_{b,:}),
\end{split}
\end{equation}
where $\overset{t}{\rightarrow}$ denotes a $t$-step shift applied to $\mathcal{X}_{a,:}$, and $\textit{corr}(\cdot)$ represents the Pearson correlation function. The strength $\rho_{ab}$ is used to weight cross-variate aggregation, so that more strongly correlated pairs contribute more. In contrast, noisier or less reliable pairs are naturally downweighted, reducing the influence of noisy or weakly correlated variate pairs. Since $\tau_{ab}$ is defined in the original time point scale, we map it to the token scale as an integer shift:
\begin{equation}
\Delta_{ab}=\mathrm{round}\left(\frac{\tau_{ab}}{P}\right)\in\mathbb{Z},
\label{eq:tau_to_delta}
\end{equation}
where $P$ is the number of time points per token. Here, $\Delta_{ab}>0$ indicates that series $a$ lags behind series $b$ on the token scale.

\begin{equation}
\Delta_{ab}=\mathrm{round}\left(\frac{\tau_{ab}}{P}\right)\in\mathbb{Z},
\label{eq:tau_to_delta}
\end{equation}
where $P$ is the time points per token, $\Delta_{ab}>0$ indicates that series $a$ lags behind series $b$ on the token scale.

\noindent\textbf{\textit{Token-level relative delay.}}
For a query token $\hat{\mathbf{X}}_{a,l}^{q}$ at position $l$ and a key token $\hat{\mathbf{X}}_{b,j}$ at position $j$, RoPE parameterizes their token-wise relative delay $(l-j)$ as
\begin{equation}
\mathcal{R}_{l-j}^{\Theta}=\mathcal{R}_l^{\Theta}(\mathcal{R}_j^{\Theta})^T,
\end{equation}
where $\mathcal{R}_l^{\Theta}$ and $\mathcal{R}_j^{\Theta}$ denote rotary position embeddings for the $l$-th and $j$-th tokens, respectively, and $\Theta$ is the shared predefined RoPE parameterization~\cite{su2024roformer}. To incorporate the global token shift, we align the key token $\hat{\mathbf{X}}_{b,j}$ from series $b$ to the effective position $j+\Delta_{ab}$ when interacting with series $a$. Consequently, the query-key interaction is governed by a unified effective delay:
\begin{equation}
\delta^{a\leftarrow b}_{l,j}=l-(j+\Delta_{ab})=(l-j)-\Delta_{ab},
\label{eq:effective_delay}
\end{equation}
which jointly accounts for the global propagation offset and the local token-wise lag.

\noindent\textbf{\textit{Delay-aware attention aggregation.}}
Given a query token $\hat{\mathbf{X}}_{a,l}^{q}$ at position $l$ in series $a$, we allow it to attend to key tokens $\hat{\mathbf{X}}_{b,j}^{k}$ from all variates $b\in\{1,\ldots,N\}$, while enforcing the causal constraint $j+\Delta_{ab}\le l$ to avoid information leakage. We inject the effective delay by applying RoPE to the query at position $l$ and to the key at position $j+\Delta_{ab}$, yielding $\mathcal{R}_l^{\Theta}$ and $\mathcal{R}_{j+\Delta_{ab}}^{\Theta}$, respectively. The delay-aware linear attention output $\mathbf{y}_{a,l}\in\mathbb{R}^{D_u}$ is computed as
\begin{align}
\label{eq:attention}
\mathbf{y}_{a,l}
=
\mathcal{R}_l^{\Theta}\phi(\hat{\mathbf{X}}_{a,l}^q)\,
\frac{
\sum_{b=1}^{N}\rho_{ab}\sum_{j:\, j+\Delta_{ab}\le l}
\big(\mathcal{R}_{j+\Delta_{ab}}^{\Theta}\phi(\hat{\mathbf{X}}_{b,j}^k)\big)^{\top}
\hat{\mathbf{X}}_{b,j}^v
}{
\phi(\hat{\mathbf{X}}_{a,l}^q)
\sum_{b=1}^{N}\rho_{ab}\sum_{j:\, j+\Delta_{ab}\le l}
\phi(\hat{\mathbf{X}}_{b,j}^k)^{\top}
},
\end{align}
where $\phi(\cdot)$ is the kernel function in~\cite{han2023flatten}, i.e., $\phi(x)=f(\mathrm{ReLU}(x))$ with $f(x)=\frac{\|x\|}{\|x^{**p}\|}x^{**p}$, and $x^{**p}$ denotes the element-wise power $p$. This kernelization enhances the expressiveness of linear attention by amplifying informative query-key similarity patterns, while $\rho_{ab}$ further emphasizes reliable cross-variate interactions. Applying Eq.~\eqref{eq:attention} to all tokens yields the DALA output
$\mathbf{Y}_{\text{variate}}\in\mathbb{R}^{L\times N\times D_u}$, which captures delay-aware cross-variate dependencies among all patched series tokens.
Overall, the robustness of DALA stems from three aspects: (i) the global delay provides coarse alignment, (ii) token-level relative delay encoding further refines interactions locally, and (iii) correlation-strength weighting suppresses unreliable variate pairs.

\textit{(c) \textbf{Gated Output Projection.}}
Following Mamba’s gated mixing mechanism, we modulate $\mathbf{Y}_{\text{variate}}$ with the gate branch $\hat{\mathbf{X}}_{\text{gate}}\in\mathbb{R}^{L\times N\times D_u}$ through an element-wise multiplicative gate, and then project it back to the model dimension:
\begin{equation}
\label{eq:mamba_dala_out}
\hat{\mathbf{Y}} = \operatorname{Linear}\big(\mathbf{Y}_{\text{variate}}\odot \sigma(\hat{\mathbf{X}}_{\text{gate}})\big),
\end{equation}
where $\odot$ denotes element-wise multiplication and $\sigma(\cdot)$ is the gating activation. Therefore, the overall Mamba-DALA module can be summarized as
\begin{equation}
\label{eq:mdala-out-batch}
\hat{\mathbf{Y}}=\operatorname{Mamba-DALA}(\hat{\mathbf{X}}).
\end{equation}

\subsubsection{\textbf{DuoMNet Block Design}}

As illustrated in Fig.~\ref{workflow}, each \textbf{DuoMNet} block consists of two complementary paths:
(i) a \emph{temporal path} implemented by Mamba-SSD to capture cross-time dependencies, and
(ii) a \emph{variate path} implemented by Mamba-DALA to model delay-aware cross-variate dependencies. We update the input $\widetilde{\mathbf{X}}^{b+1}$ and $\hat{\mathbf{X}}^{b+1}$ of both paths via residual connections and pass them to the next block:
\begin{equation}
\label{eq:duomnet_state_update}
\begin{split}
\widetilde{\mathbf{X}}^{b+1} &= \widetilde{\mathbf{X}}^{b} - \widetilde{\mathbf{Y}}^{b}=\widetilde{\mathbf{X}}^{b} -\operatorname{Mamba\mbox{-}SSD}(\widetilde{\mathbf{X}}^{b}),\\
\hat{\mathbf{X}}^{b+1} &= \hat{\mathbf{X}}^{b} - \hat{\mathbf{Y}}^{b}=\hat{\mathbf{X}}^{b} -\operatorname{Mamba\mbox{-}DALA}(\hat{\mathbf{X}}^{b}).
\end{split}
\end{equation}

To integrate the two complementary outputs, we first apply LayerNorm to each path and then perform a weighted fusion:
\begin{equation}
\label{eq:duomnet_mix}
\mathbf{U}^{b} = \alpha\,\operatorname{LN}(\widetilde{\mathbf{Y}}^{b}) + \beta\,\operatorname{LN}(\hat{\mathbf{Y}}^{b}),
\end{equation}
where $\alpha$ and $\beta$ are mixing hyperparameters.
We then apply a feed-forward network (FFN) to the mixed representation and use a residual connection followed by LN to obtain the block output:
\begin{equation}
\label{eq:duomnet_out}
\mathbf{Z}^{b} = \operatorname{LN}\!\Big(\mathbf{U}^{b} + \operatorname{FFN}(\mathbf{U}^{b})\Big).
\end{equation}
Here, $\mathbf{Z}^{b}$ is the output of the $b$-th DuoMNet block. By stacking $B$ DuoMNet blocks, we obtain $\{\mathbf{Z}^{b}\}_{b=1}^{B}$.
Following a layer-wise aggregation strategy, the final representation is computed as:
\begin{equation}
\label{eq:duomnet_final}
\mathbf{Z}=\sum_{b=1}^{B}\mathbf{Z}^{b}.
\end{equation}

\subsection{\textbf{Task-specific Projection}}

To adapt \ourname{} to diverse downstream tasks (e.g., forecasting, imputation, anomaly detection, and classification), we attach a lightweight task head on top of the final task-agnostic representation $\mathbf{Z}\in\mathbb{R}^{N\times L\times D}$ to learn task-specific targets.
For regression-oriented tasks such as forecasting, imputation, and anomaly detection, we employ an MLP as the task head
$\mathcal{Y}=\operatorname{MLP}(\mathbf{Z})$.
For forecasting, the prediction is $\mathcal{Y}\in\mathbb{R}^{N\times S}$, where $S$ is the prediction horizon.
For imputation and anomaly detection, the head produces point-wise outputs over the lookback window, i.e., $\mathcal{Y}\in\mathbb{R}^{N\times T}$, where $T$ is the lookback length.
For series-level classification, we first aggregate token-wise representations into a global feature, and then apply an MLP classifier with softmax 
$\mathcal{Y}=\operatorname{softmax}\big(\operatorname{MLP}(\mathbf{Z})\big)$,
where $\mathcal{Y}\in\mathbb{R}^{1\times C}$ and $C$ is the number of classes.

\subsection{\textbf{Training Process and Complexity Analysis}}
The overall training procedure is summarized in Algorithm~\ref{algo:dema}. 
The computational cost of \ourname{} is dominated by two parallel modules, Mamba-SSD and Mamba-DALA. Let $L$ be the number of patches and $D$ the model dimension. For Mamba-SSD, the Cross-Time scan yields $\widetilde{\mathbf{X}}\in\mathbb{R}^{N\times L\times D}$, which can be viewed as a scanned sequence of length $NL$ for SSD-based selective state-space computation. The resulting FLOPs scale linearly with the scanned length and quadratically with the feature dimension, giving a per-block complexity of $\mathcal{O}((NL)D^{2})$. Similarly, Mamba-DALA operates on $\hat{\mathbf{X}}\in\mathbb{R}^{L\times N\times D}$, which can equivalently be viewed as a sequence of length $LN$. Its core delay-aware linear mixing is implemented via linear attention and thus also requires $\mathcal{O}((LN)D^{2})$ FLOPs per block. Since the two paths are computed in parallel, the overall complexity per DuoMNet block remains $\mathcal{O}\big((NL)D^{2}\big)$. Consequently, \ourname{} achieves linear-time complexity with respect to the effective scanned length $NL$, enabling efficient modeling for long-horizon and large-scale multivariate time series in regimes where the effective token length dominates the feature dimension, i.e., $N, L \gg D$. Therefore, although DeMa contains multiple modules, its dominant computation remains concentrated in two linear-time paths, making the overall architecture computationally practical.

\begin{algorithm}[!h]
\caption{\textbf{DeMa}}
\label{algo:dema}
\KwIn{
Raw MTS window $\mathcal{X}\in\mathbb{R}^{N\times T}$;\ top-ratio $\theta$;\ patch length $P$;\ \#blocks $B$;\ fusion weights $(\alpha,\beta)$;\ task head $\text{MLP}(\cdot)$.
}
\KwOut{Task output $\mathcal{Y}$ and model $\mathcal{F}_\Theta$.}

\tcc{Adaptive Fourier Filter (AFF)}
$(\widetilde{\mathcal{X}},\hat{\mathcal{X}})\leftarrow \mathrm{AFF}(\mathcal{X};\theta)$ (Eq.~\eqref{eq:aff});

\tcc{Tokenization + Scan}
$\widetilde{\mathbf{X}}^{1}\leftarrow \mathrm{CrossTimeScan}(\mathrm{PatchEmbed}(\widetilde{\mathcal{X}};P))\in\mathbb{R}^{N\times L\times D}$;\\
$\hat{\mathbf{X}}^{1}\leftarrow \mathrm{CrossVariateScan}(\mathrm{PatchEmbed}(\hat{\mathcal{X}};P))\in\mathbb{R}^{L\times N\times D}$;

\tcc{Delay priors}
Compute $(\tau_{ab},\rho_{ab})$ by Eq.~\eqref{mcc} and $\Delta_{ab}$ by Eq.~\eqref{eq:tau_to_delta} for all series pair $(a,b)$;

\tcc{Stacked DuoMNet Blocks}
\For{$b=1,2,\ldots,B$}{
    \tcc{Parallel paths (Temporal \& Variate)}
    $\widetilde{\mathbf{Y}}^{b}\leftarrow \mathrm{Mamba\mbox{-}SSD}(\widetilde{\mathbf{X}}^{b})$
    (Eqs.~\eqref{eq:mssd-params-batch}-~\eqref{eq:mssd-out-batch});\\
    $\hat{\mathbf{Y}}^{b}\leftarrow \mathrm{Mamba\mbox{-}DALA}(\hat{\mathbf{X}}^{b};\{\Delta_{ab}\},\{\rho_{ab}\})$
    (Eqs.~\eqref{eq:effective_delay}- ~\eqref{eq:mdala-out-batch});

    \tcc{Residual update}
    $(\widetilde{\mathbf{X}}^{b+1},\hat{\mathbf{X}}^{b+1})
    \leftarrow(\widetilde{\mathbf{X}}^{b}-\widetilde{\mathbf{Y}}^{b},\ \hat{\mathbf{X}}^{b}-\hat{\mathbf{Y}}^{b})$
    (Eq.~\eqref{eq:duomnet_state_update});

    \tcc{Fusion + FFN}
    $\mathbf{U}^{b}\leftarrow \alpha\,\mathrm{LN}(\widetilde{\mathbf{Y}}^{b})+\beta\,\mathrm{LN}(\hat{\mathbf{Y}}^{b})$
    (Eq.~\eqref{eq:duomnet_mix});\\
    $\mathbf{Z}^{b}\leftarrow \mathrm{LN}\big(\mathbf{U}^{b}+\mathrm{FFN}(\mathbf{U}^{b})\big)$
    (Eq.~\eqref{eq:duomnet_out});
}

\tcc{Aggregation + Task head}
$\mathbf{Z}\leftarrow \sum_{b=1}^{B}\mathbf{Z}^{b}$ (Eq.~\eqref{eq:duomnet_final});\\
$\mathcal{Y}\leftarrow \text{MLP}(\mathbf{Z})$;

\Return{$\mathcal{Y},\ \mathcal{F}_\Theta$}
\end{algorithm}

\section{Experiments}
\label{sec:evaluation}

\begin{table*}[t]
\caption{Summary of experimental datasets. \#Variates denotes the number of variables in each multivariate time series.}
\label{tab:benchmarks}
\setlength{\abovecaptionskip}{0cm}
\setlength{\belowcaptionskip}{0cm}
\small
\renewcommand{\multirowsetup}{\centering}
\setlength{\tabcolsep}{2.2pt}
\renewcommand{\arraystretch}{0.76}
\begin{tabularx}{\textwidth}{
>{\raggedright\arraybackslash}p{0.15\textwidth}
>{\raggedright\arraybackslash}p{0.16\textwidth}
>{\centering\arraybackslash}p{0.08\textwidth}
>{\raggedright\arraybackslash}X
>{\centering\arraybackslash}p{0.10\textwidth}
>{\raggedright\arraybackslash}p{0.15\textwidth}}
\toprule
Task & Dataset & \#Variates &  (Train, Validation, Test) Size & Frequency & Domain \\
\midrule

\multirow{5}{*}{Long-term }
& ETTh1, ETTh2      & 7   & (8545, 2881, 2881)     & 1 hour  & Electricity \\
& ETTm1, ETTm2      & 7   & (34465, 11521, 11521)  & 15 min  & Electricity \\
& Electricity       & 321 & (18317, 2633, 5261)    & 1 hour  & Electricity \\
Forecasting& Weather           & 21  & (36792, 5271, 10540)   & 10 min  & Environment \\
& Traffic           & 862 & (12185, 1757, 3509)    & 1 hour  & Transportation \\
& Exchange          & 8   & (5120, 665, 1422)      & 1 day   & Economic \\
& Solar-Energy      & 137 & (36601, 5161, 10417)   & 10 min  & Energy \\
\midrule

\multirow{3}{*}{Short-term }
& PEMS03            & 358 & (15617, 5135, 5135)    & 5 min   & Transportation \\
& PEMS04            & 307 & (10172, 3375, 3375)    & 5 min   & Transportation \\
Forecasting& PEMS07            & 883 & (16911, 5622, 5622)    & 5 min   & Transportation \\
& PEMS08            & 170 & (10690, 3548, 3548)    & 5 min   & Transportation \\
\midrule

\multirow{4}{*}{Imputation}
& ETTh1, ETTh2      & 7   & (8545, 2881, 2881)     & 1 hour  & Electricity \\
& ETTm1, ETTm2      & 7   & (34465, 11521, 11521)  & 15 min  & Electricity \\
& Electricity       & 321 & (18317, 2633, 5261)    & 1 hour  & Electricity \\
& Weather           & 21  & (36792, 5271, 10540)   & 10 min  & Environment \\
\midrule

\multirow{3}{*}{Anomaly }
& SMD   & 38 & (566724, 141681, 708420) & 1 min & Server Machine \\
& MSL   & 55 & (44653, 11664, 73729)    & --    & Spacecraft \\
Detection& SMAP  & 25 & (108146, 27037, 427617)  & --    & Spacecraft \\
& SWaT  & 51 & (396000, 99000, 449919)  & 1 sec & Infrastructure \\
& PSM   & 25 & (105984, 26497, 87841)   & --    & Server Machine \\
\midrule

\multirow{10}{*}{Classification}
& EthanolConcentration & 3   & (261, 0, 263)    & -- & Chemistry \\
& FaceDetection        & 144 & (5890, 0, 3524)  & -- & Neuroscience \\
& Handwriting          & 3   & (150, 0, 850)    & -- & Motion \\
& Heartbeat            & 61  & (204, 0, 205)    & -- & Health \\
& JapaneseVowels       & 12  & (270, 0, 370)    & -- & Speech \\
& PEMS-SF              & 963 & (267, 0, 173)    & -- & Transportation \\
& SelfRegulationSCP1   & 6   & (268, 0, 293)    & -- & Health \\
& SelfRegulationSCP2   & 7   & (200, 0, 180)    & -- & Health \\
& SpokenArabicDigits   & 13  & (6599, 0, 2199)  & -- & Speech \\
& UWaveGestureLibrary  & 3   & (120, 0, 320)    & -- & Gesture \\
\bottomrule
\end{tabularx}
\end{table*}

In this section, we investigate the following research questions to validate the effectiveness, efficiency, and scalability of \ourname{}:

\begin{itemize}[leftmargin=*]
\item \textbf{RQ1}: How does \ourname{} perform across five representative time series tasks?
\item \textbf{RQ2}: Does \ourname{} achieve a favorable accuracy-efficiency trade-off as a general backbone for multivariate time-series analysis?
\item \textbf{RQ3}: How does \ourname{} scale with increasing input length in terms of per-iteration runtime and GPU memory, compared with Transformer-based baselines?
\item \textbf{RQ4}: How does each key component contribute to the overall performance of \ourname{}?
\item \textbf{RQ5}: How sensitive is \ourname{} to major hyperparameters?
\end{itemize}

\subsection{\textbf{Experimental Setup}}
\subsubsection{\textbf{Datasets}} 
To validate the effectiveness of \ourname{}, we conduct extensive experiments across five mainstream time-series analysis tasks: long-term forecasting, short-term forecasting, imputation, classification, and anomaly detection. All datasets are drawn from the benchmark collections in the Time Series Library (TSLib)~\cite{wang2024deep}. Table~\ref{tab:benchmarks} summarizes the datasets used for each task.

\begin{itemize}[leftmargin=*]
    \item \textit{Forecasting.} We evaluate both long-term and short-term forecasting. 
    For long-term forecasting, we adopt widely used benchmarks including Electricity~\cite{li2019enhancing}, ETT with four subsets (ETTh1, ETTh2, ETTm1, ETTm2)~\cite{Informer}, Exchange~\cite{LSTNet}, Traffic~\cite{wu2021autoformer}, Weather~\cite{wu2021autoformer}, and Solar-Energy~\cite{LSTNet}. 
    For short-term forecasting, we use four traffic datasets from the PEMS family~\cite{chen2001freeway}: PEMS03, PEMS04, PEMS07, and PEMS08.
    
    \item \textit{Imputation.} Time-series imputation aims to recover missing values from contextual observations. We evaluate on ETTh1, ETTh2, ETTm1, ETTm2, Electricity, and Weather. Following TimesNet~\cite{Timesnet}, we randomly mask time points with ratios in $\{12.5\%, 25\%, 37.5\%, 50\%\}$ to assess robustness under varying missing rates.
    
    \item \textit{Anomaly Detection.} Anomaly detection aims to identify abnormal patterns in time series. We adopt five widely used industrial benchmarks: Server Machine Dataset (SMD)~\cite{su2019robust}, Mars Science Laboratory rover (MSL)~\cite{hundman2018detecting}, Soil Moisture Active Passive satellite (SMAP)~\cite{hundman2018detecting}, Secure Water Treatment (SWaT)~\cite{mathur2016swat}, and Pooled Server Metrics (PSM)~\cite{abdulaal2021practical}.
    
    \item \textit{Classification.} For time-series classification, we use ten multivariate datasets from the UEA Time Series Classification Archive~\cite{bagnall2018uea}, covering diverse domains such as gesture recognition, action recognition, audio recognition, and medical diagnosis.
\end{itemize}

\subsubsection{\textbf{Baselines}}
To ensure a comprehensive comparison, we include a broad set of strong baselines that are representative of the latest advances in the time series community. Specifically, our baselines span four families:
(1) Mamba-based models: Affirm~\cite{wu2025affirm}, S-Mamba~\cite{wang2024mamba}, CMamba~\cite{cmamba}, and SAMBA~\cite{samba};
(2) Transformer-based models: iTransformer~\cite{liuitransformer}, Crossformer~\cite{Crossformer}, and PatchTST~\cite{PatchTST};
(3) MLP-based models: TimeMixer~\cite{timemixer}, DLinear~\cite{DLinear}, and RLinear~\cite{li2023revisiting}; and
(4) TCN-based models: ModernTCN~\cite{luo2024moderntcn} and TimesNet~\cite{Timesnet}.

\subsection{\textbf{Metrics and Implementation Details}}
To ensure a fair and comprehensive comparison, we follow the experimental protocol of the well-established Time Series Library~\cite{wang2024deep}. 
For long-term forecasting and imputation, we report mean squared error (MSE) and mean absolute error (MAE). 
For time-series classification, we report accuracy. 
For anomaly detection, we report the F1-score, which balances precision and recall and is well-suited to the highly imbalanced nature of anomalous events. 
In each table, the best and second-best results are highlighted in \boldres{red} and \secondres{underlined}, respectively.

We use the published optimal hyperparameter settings for all baselines. All experiments are implemented in PyTorch and conducted on three NVIDIA RTX A6000 GPUs (48GB). We train all models with the Adam optimizer~\cite{Adam} under an $\ell_2$ loss, and tune the initial learning rate in $\{1\mathrm{e}{-4},,5\mathrm{e}{-4},,1\mathrm{e}{-3}\}$. The number of DuoMNet blocks is selected from $\{2,3,4,5\}$ via hyperparameter search, and the representation dimension is chosen from $\{32,64,128,256\}$. We set the batch size to 32 and train for 50 epochs. For the Adaptive Fourier Filter, the top-ratio $\theta$ is selected from a limited candidate set, and we use $\theta=60\%$ in the final configuration. We further conduct a grid search over the fusion weights $\alpha,\beta\in\{0.2,0.4,0.5,0.6,0.8\}$, and use $(\alpha,\beta)=(0.6,0.4)$ for forecasting and classification, and $(\alpha,\beta)=(0.2,0.8)$ for imputation and anomaly detection. For tokenization, we fix the patch size and stride to $P=S=8$. DeMa follows a modular pipeline. Once the temporal path (Mamba-SSD) and the variate path (Mamba-DALA) are instantiated, the model mainly repeats the same DuoMNet block, which helps keep the implementation structured and reproducible.

\begin{table*}[!h]
  \caption{  
  Long-term forecasting results with lookback length $T=96$ and prediction length $S\in\{96, 192, 336, 720\}$. Lower MSE or MAE is better.}
  \label{tab:long_term_result}
  \renewcommand{\arraystretch}{0.85} 
  \centering
  \resizebox{\textwidth}{!}{
  \begin{threeparttable}
  \begin{small}
  \renewcommand{\multirowsetup}{\centering}
  \setlength{\tabcolsep}{1.45pt}
\begin{tabular}{cccccccccccccccccccccccccccc}
    \toprule
    \multicolumn{2}{c}{\multirow{3}{*}{Models}} &
    \multicolumn{2}{c}{\scalebox{0.9}{\textbf{SSD + DALA}}} &
    \multicolumn{8}{c}{\scalebox{0.8}{SSM-based}} &
    \multicolumn{6}{c}{\scalebox{0.8}{Transformer-based models}} &
    \multicolumn{6}{c}{\scalebox{0.8}{MLP-based models}} &
    \multicolumn{4}{c}{\scalebox{0.8}{TCN-based models}} \\
    \cmidrule(lr){3-4} \cmidrule(lr){5-12} \cmidrule(lr){13-18}\cmidrule(lr){19-24}\cmidrule(lr){25-28}

    \multicolumn{2}{c}{} & 
    \multicolumn{2}{c}{\rotatebox{0}{\scalebox{0.9}{\textbf{\ourname{}}}}}&
    
    \multicolumn{2}{c}{\rotatebox{0}{\scalebox{0.8}{Affirm}}} &
    \multicolumn{2}{c}{\rotatebox{0}{\scalebox{0.8}{S-Mamba}}} &
    \multicolumn{2}{c}{\rotatebox{0}{\scalebox{0.8}{CMamba}}} &
    \multicolumn{2}{c}{\rotatebox{0}{\scalebox{0.8}{SAMBA}}} &

    % \multicolumn{2}{c}{\rotatebox{0}{\scalebox{0.8}{FEDformer}}} &
    \multicolumn{2}{c}{\rotatebox{0}{\scalebox{0.8}{iTransformer}}} &
    \multicolumn{2}{c}{\rotatebox{0}{\scalebox{0.8}{Crossformer}}} &
    \multicolumn{2}{c}{\rotatebox{0}{\scalebox{0.8}{PatchTST}}} &

    % \multicolumn{2}{c}{\rotatebox{0}{\scalebox{0.8}{TiDE}}} &
    \multicolumn{2}{c}{\rotatebox{0}{\scalebox{0.8}{TimeMixer}}} &
    \multicolumn{2}{c}{\rotatebox{0}{\scalebox{0.8}{{RLinear}}}} &
    \multicolumn{2}{c}{\rotatebox{0}{\scalebox{0.8}{DLinear}}} &

    % \multicolumn{2}{c}{\rotatebox{0}{\scalebox{0.8}{SCINet}}} &
    \multicolumn{2}{c}{\rotatebox{0}{\scalebox{0.8}{ModernTCN}}} &
    \multicolumn{2}{c}{\rotatebox{0}{\scalebox{0.8}{TimesNet}}}
    \\
\bottomrule
  \end{tabular}
    \end{small}
  \end{threeparttable}
}
\end{table*} 
\subsection{\textbf{Main Results Across Five Time-Series Tasks (RQ1)}}

\subsubsection{\textbf{Long-term Forecasting Results}}
Time-series forecasting is a fundamental task in time-series analysis, aiming to predict future values from historical observations. Table~\ref{tab:long_term_result} reports the long-term forecasting results with a fixed lookback length $T=96$ and prediction horizons $S\in\{96,192,336,720\}$. Overall, \ourname{} delivers top-tier accuracy consistently across all datasets and horizons, ranking within the top two in every setting. In particular, \ourname{} achieves the best results in 34/36 cases for MSE and 30/36 cases for MAE, demonstrating strong robustness and generalization across diverse temporal dynamics.

Compared with Transformer-based models such as iTransformer, PatchTST, and Crossformer, \ourname{} consistently performs better across all settings, indicating that explicitly disentangling and modeling temporal and variate dependencies is more effective for long-horizon forecasting than relying primarily on global attention. \ourname{} also consistently surpasses Mamba variants and non-attention baselines. For example, on Traffic (Avg), \ourname{} reduces the MSE from 0.414 to 0.372, achieving a 10.1\% improvement over the strongest Mamba baseline, S-Mamba, and also outperforms representative MLP- and TCN-based models, including TimeMixer (0.484) and ModernTCN (0.546). These gains can be attributed to DeMa’s comprehensive dependency modeling: Mamba-SSD efficiently captures long-range temporal dependencies through parallel selective state-space computation, while Mamba-DALA enhances cross-variate interaction modeling. Their synergy enables \ourname{} to exploit long-range temporal structures and inter-series correlations jointly, yielding consistently strong performance in long-term forecasting. 
We note that DeMa remains highly competitive even on datasets with only a few variates. As summarized in Table~\ref{tab:benchmarks}, ETTh1, ETTh2, ETTm1, and ETTm2 each contain 7 variates, while Exchange contains 8 variates. On these small-scale benchmarks, DeMa still achieves best or near-best averaged results in Table~\ref{tab:long_term_result}. These results indicate that the benefit of DeMa persists in small-scale scenarios, though its relative advantage becomes more pronounced as scalability pressure increases.

We further examine DeMa under relatively weak inter-variable dependencies, as analyzed on ETTh2, ETTm1, and ETTm2. Even in these settings, DeMa remains highly competitive, achieving the best average MSE on ETTm1 (0.367) and ETTm2 (0.270), while maintaining comparable performance on ETTh2 (0.370). This observation suggests that when the correlation-based delay prior becomes less informative, the temporal path can still capture the dominant intra-series dynamics, whereas the variate path provides complementary gains.
\begin{table*}[!h]
  \caption{Short-term forecasting results with lookback length $T=96$ and prediction length $S\in\{12,24,48\}$. Lower MSE or MAE is better.}
  \label{tab:short_term_results}
  \renewcommand{\arraystretch}{0.85} 
  \centering
  \resizebox{\textwidth}{!}{
  \begin{threeparttable}
  \begin{small}
  \renewcommand{\multirowsetup}{\centering}
  \setlength{\tabcolsep}{1pt}
  \begin{tabular}{cccccccccccccccccccccccccccc}
    \toprule
    \multicolumn{2}{c}{\multirow{3}{*}{Models}} &
    \multicolumn{2}{c}{\scalebox{0.9}{\textbf{SSD + DALA}}} &
    \multicolumn{8}{c}{\scalebox{0.8}{SSM-based}} &
    \multicolumn{6}{c}{\scalebox{0.8}{Transformer-based models}} &
    \multicolumn{6}{c}{\scalebox{0.8}{MLP-based models}} &
    \multicolumn{4}{c}{\scalebox{0.8}{TCN-based models}} \\ 
    \cmidrule(lr){3-4} \cmidrule(lr){5-12} \cmidrule(lr){13-18}\cmidrule(lr){19-24}\cmidrule(lr){25-28}

    \multicolumn{2}{c}{} & 
    \multicolumn{2}{c}{\rotatebox{0}{\scalebox{0.9}{\textbf{\ourname{}}}}}&
    
    \multicolumn{2}{c}{\rotatebox{0}{\scalebox{0.8}{Affirm}}} &
    \multicolumn{2}{c}{\rotatebox{0}{\scalebox{0.8}{S-Mamba}}} &
    \multicolumn{2}{c}{\rotatebox{0}{\scalebox{0.8}{CMamba}}} &
    \multicolumn{2}{c}{\rotatebox{0}{\scalebox{0.8}{SAMBA}}} &

    % \multicolumn{2}{c}{\rotatebox{0}{\scalebox{0.8}{FEDformer}}} &
    \multicolumn{2}{c}{\rotatebox{0}{\scalebox{0.8}{iTransformer}}} &
    \multicolumn{2}{c}{\rotatebox{0}{\scalebox{0.8}{Crossformer}}} &
    \multicolumn{2}{c}{\rotatebox{0}{\scalebox{0.8}{PatchTST}}} &

    % \multicolumn{2}{c}{\rotatebox{0}{\scalebox{0.8}{TiDE}}} &
    \multicolumn{2}{c}{\rotatebox{0}{\scalebox{0.8}{TimeMixer}}} &
    \multicolumn{2}{c}{\rotatebox{0}{\scalebox{0.8}{{RLinear}}}} &
    \multicolumn{2}{c}{\rotatebox{0}{\scalebox{0.8}{DLinear}}} &

    % \multicolumn{2}{c}{\rotatebox{0}{\scalebox{0.8}{SCINet}}} &
    \multicolumn{2}{c}{\rotatebox{0}{\scalebox{0.8}{ModernTCN}}} &
    \multicolumn{2}{c}{\rotatebox{0}{\scalebox{0.8}{TimesNet}}}
    \\
     
    \multicolumn{2}{c}{} &
    \multicolumn{2}{c}{\scalebox{0.9}{\textbf{(Ours)}}}&

    \multicolumn{2}{c}{\scalebox{0.8}{\cite{wu2025affirm}}} &
    \multicolumn{2}{c}{\scalebox{0.8}{\cite{wang2024mamba}}} &
    \multicolumn{2}{c}{\scalebox{0.8}{\cite{cmamba}}} &
    \multicolumn{2}{c}{\scalebox{0.8}{\cite{samba}}} &

    % \multicolumn{2}{c}{\scalebox{0.8}{\cite{fedformer}}} &
    \multicolumn{2}{c}{\scalebox{0.8}{\cite{liuitransformer}}} &
    \multicolumn{2}{c}{\scalebox{0.8}{\cite{Crossformer}}} & 
    \multicolumn{2}{c}{\scalebox{0.8}{\cite{PatchTST}}} &

    % \multicolumn{2}{c}{\scalebox{0.8}{\cite{daslong}}} & 
    \multicolumn{2}{c}{\scalebox{0.8}{\cite{timemixer}}} & 
    \multicolumn{2}{c}{\scalebox{0.8}{\cite{li2023revisiting}}} & 
    \multicolumn{2}{c}{\scalebox{0.8}{\cite{DLinear}}} &

    % \multicolumn{2}{c}{\scalebox{0.8}{\cite{SCINet}}} & 
    \multicolumn{2}{c}{\scalebox{0.8}{\cite{luo2024moderntcn}}} &
    \multicolumn{2}{c}{\scalebox{0.8}{\cite{Timesnet}}}
    \\
    
   \cmidrule(lr){3-4} \cmidrule(lr){5-6}\cmidrule(lr){7-8} \cmidrule(lr){9-10}\cmidrule(lr){11-12}\cmidrule(lr){13-14} \cmidrule(lr){15-16} \cmidrule(lr){17-18} \cmidrule(lr){19-20} \cmidrule(lr){21-22} \cmidrule(lr){23-24} \cmidrule(lr){25-26} \cmidrule(lr){27-28} 
    
    \multicolumn{2}{c}{Metric}  
    & \scalebox{0.8}{MSE} & \scalebox{0.8}{MAE} 
    & \scalebox{0.8}{MSE} & \scalebox{0.8}{MAE} 
    & \scalebox{0.8}{MSE} & \scalebox{0.8}{MAE} 
    & \scalebox{0.8}{MSE} & \scalebox{0.8}{MAE} 
    & \scalebox{0.8}{MSE} & \scalebox{0.8}{MAE} 
    & \scalebox{0.8}{MSE} & \scalebox{0.8}{MAE} 
    & \scalebox{0.8}{MSE} & \scalebox{0.8}{MAE} 
    & \scalebox{0.8}{MSE} & \scalebox{0.8}{MAE} 
    & \scalebox{0.8}{MSE} & \scalebox{0.8}{MAE} 
    & \scalebox{0.8}{MSE} & \scalebox{0.8}{MAE} 
    & \scalebox{0.8}{MSE} & \scalebox{0.8}{MAE} 
    & \scalebox{0.8}{MSE} & \scalebox{0.8}{MAE} 
    & \scalebox{0.8}{MSE} & \scalebox{0.8}{MAE} \\
    \toprule
    
    \multirow{4}{*}{\rotatebox{90}{\scalebox{0.95}{PEMS03}}}
    &\scalebox{0.8}{12} 
    % \Our
    &\boldres{\scalebox{0.8}{0.050}} &\boldres{\scalebox{0.8}{0.146}}
    % \SSM
    &\scalebox{0.8}{0.091} &\scalebox{0.8}{0.200}
    &\secondres{\scalebox{0.8}{0.065}} &\secondres{\scalebox{0.8}{0.169}}
    &\scalebox{0.8}{0.068} &\scalebox{0.8}{0.172}
    &\scalebox{0.8}{0.066} &\scalebox{0.8}{0.182}

    &\scalebox{0.8}{0.071} &\scalebox{0.8}{0.174} 
    &\scalebox{0.8}{0.090} &\scalebox{0.8}{0.203} 
    &\scalebox{0.8}{0.099} &\scalebox{0.8}{0.216}

    &\scalebox{0.8}{0.073} &\scalebox{0.8}{0.181}
    &\scalebox{0.8}{0.126} &\scalebox{0.8}{0.236} 
    &\scalebox{0.8}{0.122} &\scalebox{0.8}{0.243}

    &\scalebox{0.8}{0.072} & \scalebox{0.8}{0.186} 
    &\scalebox{0.8}{0.085} &\scalebox{0.8}{0.192} 
    \\
    
    & \scalebox{0.8}{24} 
    % \Our
    &\boldres{\scalebox{0.8}{0.063}} &\boldres{\scalebox{0.8}{0.172}}
    % \SSM
    &\scalebox{0.8}{0.104} &\scalebox{0.8}{0.228}
    &\scalebox{0.8}{0.087} &\secondres{\scalebox{0.8}{0.196}} 
    &\scalebox{0.8}{0.079} &\scalebox{0.8}{0.201}
    &\secondres{\scalebox{0.8}{0.075}} &\scalebox{0.8}{0.203}
    
    &\scalebox{0.8}{0.093} &\scalebox{0.8}{0.201} 
    &\scalebox{0.8}{0.121} &\scalebox{0.8}{0.240} 
    &\scalebox{0.8}{0.142} &\scalebox{0.8}{0.259} 

    &\scalebox{0.8}{0.091} &\scalebox{0.8}{0.232} 
    &\scalebox{0.8}{0.246} &\scalebox{0.8}{0.334}
    &\scalebox{0.8}{0.201} &\scalebox{0.8}{0.317} 

    &\scalebox{0.8}{0.095} &\scalebox{0.8}{0.217} 
    &\scalebox{0.8}{0.118} &\scalebox{0.8}{0.223} 
     \\
    
    & \scalebox{0.8}{48} 
    % \Our
    &\boldres{\scalebox{0.8}{0.093}} &\boldres{\scalebox{0.8}{0.204}}
    % \SSM
    &\scalebox{0.8}{0.199} &\scalebox{0.8}{0.281}
    &\scalebox{0.8}{0.133} &\scalebox{0.8}{0.243} 
    &\scalebox{0.8}{0.132} &\scalebox{0.8}{0.239} 
    &\scalebox{0.8}{0.147} &\scalebox{0.8}{0.255} 
   
    &\secondres{\scalebox{0.8}{0.125}} &\secondres{\scalebox{0.8}{0.236}}
    &\scalebox{0.8}{0.202} &\scalebox{0.8}{0.317} 
    &\scalebox{0.8}{0.211} &\scalebox{0.8}{0.319}
    
    &\scalebox{0.8}{0.138} &\scalebox{0.8}{0.244} 
    &\scalebox{0.8}{0.551} &\scalebox{0.8}{0.529}
    &\scalebox{0.8}{0.333} &\scalebox{0.8}{0.425}
    
    &\scalebox{0.8}{0.136} &\scalebox{0.8}{0.249}
    &\scalebox{0.8}{0.155} &\scalebox{0.8}{0.260}
    
    \\
    \cmidrule(lr){2-28}

    & \scalebox{0.8}{Avg} 
    % \Our
    &\boldres{\scalebox{0.8}{0.069}} &\boldres{\scalebox{0.8}{0.174}}
    % \SSM
    &\scalebox{0.8}{0.131} &\scalebox{0.8}{0.236}
    &\scalebox{0.8}{0.095} &\secondres{\scalebox{0.8}{0.203}}
    &\secondres{\scalebox{0.8}{0.093}} &\scalebox{0.8}{0.204}
    &\scalebox{0.8}{0.096} &\scalebox{0.8}{0.213}
    
    &\scalebox{0.8}{0.096} &\scalebox{0.8}{0.204}
    &\scalebox{0.8}{0.138} &\scalebox{0.8}{0.253}
    &\scalebox{0.8}{0.151} &\scalebox{0.8}{0.265}
    
    &\scalebox{0.8}{0.101} &\scalebox{0.8}{0.219} 
    &\scalebox{0.8}{0.308} &\scalebox{0.8}{0.366}
    &\scalebox{0.8}{0.219} &\scalebox{0.8}{0.328}
    
    &\scalebox{0.8}{0.101} &\scalebox{0.8}{0.217}
    &\scalebox{0.8}{0.119} &\scalebox{0.8}{0.225}
    
    \\
    \midrule
    
    \multirow{4}{*}{\rotatebox{90}{\scalebox{0.95}{PEMS04}}}
    &\scalebox{0.8}{12} 
    % \Our
    &\boldres{\scalebox{0.8}{0.061}} &\boldres{\scalebox{0.8}{0.168}}
    % \SSM
    &\scalebox{0.8}{0.092} &\scalebox{0.8}{0.201}
    &\scalebox{0.8}{0.076} &\scalebox{0.8}{0.180} 
    &\scalebox{0.8}{0.079} &\scalebox{0.8}{0.199} 
    &\secondres{\scalebox{0.8}{0.073}} &\secondres{\scalebox{0.8}{0.172}} 
    
    &\scalebox{0.8}{0.078} &\scalebox{0.8}{0.183} 
    &\scalebox{0.8}{0.098} &\scalebox{0.8}{0.218} 
    &\scalebox{0.8}{0.105} &\scalebox{0.8}{0.224}
    
    &\scalebox{0.8}{0.082}& \scalebox{0.8}{0.201}
    &\scalebox{0.8}{0.138} &\scalebox{0.8}{0.252}
    &\scalebox{0.8}{0.148} &\scalebox{0.8}{0.272}
    
    &\scalebox{0.8}{0.082} &\scalebox{0.8}{0.192} 
    &\scalebox{0.8}{0.087} &\scalebox{0.8}{0.195} 
     \\
    
    & \scalebox{0.8}{24} 
    % \Our
    &\boldres{\scalebox{0.8}{0.062}} &\boldres{\scalebox{0.8}{0.172}}
    % \SSM
    &\scalebox{0.8}{0.119} &\scalebox{0.8}{0.256}
    &\secondres{\scalebox{0.8}{0.084}} &\scalebox{0.8}{0.193} 
    &\scalebox{0.8}{0.091} &\scalebox{0.8}{0.228} 
    &\scalebox{0.8}{0.086} &\secondres{\scalebox{0.8}{0.185}} 
    
    &\scalebox{0.8}{0.095} &\scalebox{0.8}{0.205}
    &\scalebox{0.8}{0.131} &\scalebox{0.8}{0.256} 
    &\scalebox{0.8}{0.153} &\scalebox{0.8}{0.275} 
    
    &\scalebox{0.8}{0.102} &\scalebox{0.8}{0.222}
    &\scalebox{0.8}{0.258} &\scalebox{0.8}{0.348}
    &\scalebox{0.8}{0.224} &\scalebox{0.8}{0.340} 
    
    &\scalebox{0.8}{0.093} &\scalebox{0.8}{0.209} 
    &\scalebox{0.8}{0.103} &\scalebox{0.8}{0.215} 
    
     \\
    
    & \scalebox{0.8}{48} 
    % \Our
    &\boldres{\scalebox{0.8}{0.087}} &\boldres{\scalebox{0.8}{0.193}}
    % \SSM
    &\scalebox{0.8}{0.187} &\scalebox{0.8}{0.283}
    &\scalebox{0.8}{0.115} &\scalebox{0.8}{0.224} 
    &\scalebox{0.8}{0.156} &\scalebox{0.8}{0.266} 
    &\secondres{\scalebox{0.8}{0.092}} &\secondres{\scalebox{0.8}{0.209}} 
    
    &\scalebox{0.8}{0.120} &\scalebox{0.8}{0.233} 
    &\scalebox{0.8}{0.205} &\scalebox{0.8}{0.326} 
    &\scalebox{0.8}{0.229} &\scalebox{0.8}{0.339} 
    
    &\scalebox{0.8}{0.157} &\scalebox{0.8}{0.289} 
    &\scalebox{0.8}{0.572} &\scalebox{0.8}{0.544} 
    &\scalebox{0.8}{0.355} &\scalebox{0.8}{0.437} 
    
    &\scalebox{0.8}{0.123} &\scalebox{0.8}{0.227}
    &\scalebox{0.8}{0.136} &\scalebox{0.8}{0.250} 
     \\
    \cmidrule(lr){2-28}

    & \scalebox{0.8}{Avg} 
    % \Our
    &\boldres{\scalebox{0.8}{0.070}} &\boldres{\scalebox{0.8}{0.178}}
    % \SSM
    &\scalebox{0.8}{0.133} &\scalebox{0.8}{0.247}
    &\scalebox{0.8}{0.092} &\scalebox{0.8}{0.199}
    &\scalebox{0.8}{0.109} &\scalebox{0.8}{0.231}
    &\secondres{\scalebox{0.8}{0.083}} &\secondres{\scalebox{0.8}{0.189}}
    
    &\scalebox{0.8}{0.098} &\scalebox{0.8}{0.207}
    &\scalebox{0.8}{0.145} &\scalebox{0.8}{0.267}
    &\scalebox{0.8}{0.162} &\scalebox{0.8}{0.279}
    
    &\scalebox{0.8}{0.114} &\scalebox{0.8}{0.238} 
    &\scalebox{0.8}{0.323} &\scalebox{0.8}{0.381}
    &\scalebox{0.8}{0.242} &\scalebox{0.8}{0.350}
    
    &\scalebox{0.8}{0.099} &\scalebox{0.8}{0.209}
    &\scalebox{0.8}{0.109} &\scalebox{0.8}{0.220}
    \\
    \midrule
    
    \multirow{4}{*}{\rotatebox{90}{\scalebox{0.95}{PEMS07}}}
    &\scalebox{0.8}{12} 
    % \Our
    &\boldres{\scalebox{0.8}{0.059}} &\boldres{\scalebox{0.8}{0.147}}
    % \SSM
    &\scalebox{0.8}{0.083} &\scalebox{0.8}{0.192}
    &\secondres{\scalebox{0.8}{0.063}} &\secondres{\scalebox{0.8}{0.159}} 
    &\scalebox{0.8}{0.069} &\scalebox{0.8}{0.172} 
    &\scalebox{0.8}{0.067} &\scalebox{0.8}{0.177} 
    
    &\scalebox{0.8}{0.067} &\scalebox{0.8}{0.165} 
    &\scalebox{0.8}{0.094} &\scalebox{0.8}{0.200} 
    &\scalebox{0.8}{0.095} &\scalebox{0.8}{0.207} 
    
    &\scalebox{0.8}{0.078} &\scalebox{0.8}{0.181} 
    &\scalebox{0.8}{0.118} &\scalebox{0.8}{0.235}
    &\scalebox{0.8}{0.115} &\scalebox{0.8}{0.242} 
    
    &\scalebox{0.8}{0.073} &\scalebox{0.8}{0.175} 
    &\scalebox{0.8}{0.082} &\scalebox{0.8}{0.181} 
     \\
    
    & \scalebox{0.8}{24} 
    % \Our
    &\boldres{\scalebox{0.8}{0.069}} &\boldres{\scalebox{0.8}{0.162}}
    % \SSM
    &\scalebox{0.8}{0.128} &\scalebox{0.8}{0.227}
    &\secondres{\scalebox{0.8}{0.081}} &\secondres{\scalebox{0.8}{0.183}} 
    &\scalebox{0.8}{0.092} &\scalebox{0.8}{0.194} 
    &\scalebox{0.8}{0.094} &\scalebox{0.8}{0.211} 
    
    &\scalebox{0.8}{0.088} &\scalebox{0.8}{0.190} 
    &\scalebox{0.8}{0.139} &\scalebox{0.8}{0.247} 
    &\scalebox{0.8}{0.150} &\scalebox{0.8}{0.262} 
    
    &\scalebox{0.8}{0.108} &\scalebox{0.8}{0.239} 
    &\scalebox{0.8}{0.242} &\scalebox{0.8}{0.341}
    &\scalebox{0.8}{0.210} &\scalebox{0.8}{0.329} 
    
    &\scalebox{0.8}{0.095} &\scalebox{0.8}{0.193} 
    &\scalebox{0.8}{0.101} &\scalebox{0.8}{0.204} 
     \\
    
    & \scalebox{0.8}{48}
    % \Our
    &\boldres{\scalebox{0.8}{0.079}} &\boldres{\scalebox{0.8}{0.174}}
    % \SSM
    &\scalebox{0.8}{0.205} &\scalebox{0.8}{0.299}
    &\secondres{\scalebox{0.8}{0.093}} &\secondres{\scalebox{0.8}{0.192}} 
    &\scalebox{0.8}{0.137} &\scalebox{0.8}{0.255} 
    &\scalebox{0.8}{0.122} &\scalebox{0.8}{0.236}
    
    &\scalebox{0.8}{0.110} &\scalebox{0.8}{0.215} 
    &\scalebox{0.8}{0.311} &\scalebox{0.8}{0.369} 
    &\scalebox{0.8}{0.253} &\scalebox{0.8}{0.340} 
    
    &\scalebox{0.8}{0.167} &\scalebox{0.8}{0.292} 
    &\scalebox{0.8}{0.562} &\scalebox{0.8}{0.541} 
    &\scalebox{0.8}{0.398} &\scalebox{0.8}{0.458}
    
    &\scalebox{0.8}{0.116} &\scalebox{0.8}{0.221} 
    &\scalebox{0.8}{0.134} &\scalebox{0.8}{0.238} 
     \\
    \cmidrule(lr){2-28}

    & \scalebox{0.8}{Avg} 
    % \Our
    &\boldres{\scalebox{0.8}{0.069}} &\boldres{\scalebox{0.8}{0.161}}
    % \SSM
    &\scalebox{0.8}{0.139} &\scalebox{0.8}{0.240}
    &\secondres{\scalebox{0.8}{0.079}} &\secondres{\scalebox{0.8}{0.178}}
    &\scalebox{0.8}{0.099} &\scalebox{0.8}{0.207}
    &\scalebox{0.8}{0.094} &\scalebox{0.8}{0.208}
    
    &\scalebox{0.8}{0.088} &\scalebox{0.8}{0.190}
    &\scalebox{0.8}{0.181} &\scalebox{0.8}{0.272}
    &\scalebox{0.8}{0.166} &\scalebox{0.8}{0.270}

    &\scalebox{0.8}{0.118} &\scalebox{0.8}{0.237} 
    &\scalebox{0.8}{0.307} &\scalebox{0.8}{0.372}
    &\scalebox{0.8}{0.241} &\scalebox{0.8}{0.343}
    
    &\scalebox{0.8}{0.095} &\scalebox{0.8}{0.196}
    &\scalebox{0.8}{0.106} &\scalebox{0.8}{0.208}    
    \\
    \midrule
    
    \multirow{4}{*}{\rotatebox{90}{\scalebox{0.95}{PEMS08}}}
    & \scalebox{0.8}{12} 
    % \Our
    &\boldres{\scalebox{0.8}{0.052}} &\boldres{\scalebox{0.8}{0.158}}
    % \SSM
    &\scalebox{0.8}{0.141} &\scalebox{0.8}{0.192}
    &\secondres{\scalebox{0.8}{0.076}} &\secondres{\scalebox{0.8}{0.178}} 
    &\scalebox{0.8}{0.082} &\scalebox{0.8}{0.199}
    &\scalebox{0.8}{0.080} &\scalebox{0.8}{0.191} 
    
    &\scalebox{0.8}{0.079} &\scalebox{0.8}{0.182}
    &\scalebox{0.8}{0.165} &\scalebox{0.8}{0.214} 
    &\scalebox{0.8}{0.168} &\scalebox{0.8}{0.232} 
    
    &\scalebox{0.8}{0.144} &\scalebox{0.8}{0.196} 
    &\scalebox{0.8}{0.133} &\scalebox{0.8}{0.247} 
    &\scalebox{0.8}{0.154} &\scalebox{0.8}{0.276} 
    
    &\scalebox{0.8}{0.082} &\scalebox{0.8}{0.181}
    &\scalebox{0.8}{0.112} &\scalebox{0.8}{0.212} 
     \\
    
    & \scalebox{0.8}{24} 
    % \Our
    &\boldres{\scalebox{0.8}{0.091}} &\boldres{\scalebox{0.8}{0.183}} 
    % \SSM
    &\scalebox{0.8}{0.183} &\scalebox{0.8}{0.299}
    &\secondres{\scalebox{0.8}{0.104}} &\secondres{\scalebox{0.8}{0.209}} 
    &\scalebox{0.8}{0.134} &\scalebox{0.8}{0.238} 
    &\scalebox{0.8}{0.116} &\scalebox{0.8}{0.221} 
    
    &\scalebox{0.8}{0.115} &\scalebox{0.8}{0.219} 
    &\scalebox{0.8}{0.215} &\scalebox{0.8}{0.260} 
    &\scalebox{0.8}{0.224} &\scalebox{0.8}{0.281} 
    
    &\scalebox{0.8}{0.181} &\scalebox{0.8}{0.272}
    &\scalebox{0.8}{0.249} &\scalebox{0.8}{0.343} 
    &\scalebox{0.8}{0.248} &\scalebox{0.8}{0.353}
    
    &\scalebox{0.8}{0.129} &\scalebox{0.8}{0.227} 
    &\scalebox{0.8}{0.141} &\scalebox{0.8}{0.238} 
    \\
    
    & \scalebox{0.8}{48} 
    % \Our
    &\boldres{\scalebox{0.8}{0.149}} &\boldres{\scalebox{0.8}{0.206}} 
    % \SSM
    &\scalebox{0.8}{0.237} &\scalebox{0.8}{0.307}
    &\secondres{\scalebox{0.8}{0.167}} &\secondres{\scalebox{0.8}{0.228}} 
    &\scalebox{0.8}{0.200} &\scalebox{0.8}{0.255} 
    &\scalebox{0.8}{0.199} &\scalebox{0.8}{0.241} 
    
    &\scalebox{0.8}{0.186} &\scalebox{0.8}{0.235} 
    &\scalebox{0.8}{0.315} &\scalebox{0.8}{0.355} 
    &\scalebox{0.8}{0.321} &\scalebox{0.8}{0.354} 
    
    &\scalebox{0.8}{0.259} &\scalebox{0.8}{0.323} 
    &\scalebox{0.8}{0.569} &\scalebox{0.8}{0.544} 
    &\scalebox{0.8}{0.440} &\scalebox{0.8}{0.470} 
    
    &\scalebox{0.8}{0.183} &\scalebox{0.8}{0.231} 
    &\scalebox{0.8}{0.198} &\scalebox{0.8}{0.283} 
    \\
    \cmidrule(lr){2-28}

    & \scalebox{0.8}{Avg} 
    % \Our
    &\boldres{\scalebox{0.8}{0.097}} &\boldres{\scalebox{0.8}{0.182}}
    % \SSM
    &\scalebox{0.8}{0.187} &\scalebox{0.8}{0.266}
    &\secondres{\scalebox{0.8}{0.116}} &\secondres{\scalebox{0.8}{0.205}}
    &\scalebox{0.8}{0.139} &\scalebox{0.8}{0.231}
    &\scalebox{0.8}{0.132} &\scalebox{0.8}{0.218}

    &\scalebox{0.8}{0.127} &\scalebox{0.8}{0.212}
    &\scalebox{0.8}{0.232} &\scalebox{0.8}{0.276}
    &\scalebox{0.8}{0.238} &\scalebox{0.8}{0.289}
    
    &\scalebox{0.8}{0.195} &\scalebox{0.8}{0.264} 
    &\scalebox{0.8}{0.317} &\scalebox{0.8}{0.378}
    &\scalebox{0.8}{0.281} &\scalebox{0.8}{0.366}

    &\scalebox{0.8}{0.131} &\scalebox{0.8}{0.213}
    &\scalebox{0.8}{0.150} &\scalebox{0.8}{0.244}
    \\
    \midrule
     \multicolumn{2}{c}{\scalebox{0.8}{{$1^{\text{st}}$ Count}}} 
     % \Our
     &\boldres{\scalebox{0.8}{16}} &\boldres{\scalebox{0.8}{16}} 
     % \SSM
     &\scalebox{0.8}{0} &\scalebox{0.8}{0}
     &\scalebox{0.8}{0} &\scalebox{0.8}{0}
     &\scalebox{0.8}{0} &\scalebox{0.8}{0}
     &\scalebox{0.8}{0} &\scalebox{0.8}{0}
     
     &\scalebox{0.8}{0} &\scalebox{0.8}{0}
     &\scalebox{0.8}{0} &\scalebox{0.8}{0}
     &\scalebox{0.8}{0} &\scalebox{0.8}{0}
     
     &\scalebox{0.8}{0} &\scalebox{0.8}{0}
     &\scalebox{0.8}{0} &\scalebox{0.8}{0} 
     &\scalebox{0.8}{0} &\scalebox{0.8}{0} 
     
     &\scalebox{0.8}{0} &\scalebox{0.8}{0}
     &\scalebox{0.8}{0} &\scalebox{0.8}{0}
     \\
        \bottomrule
      \end{tabular}
    \end{small}
  \end{threeparttable}
  }
\end{table*}

\subsubsection{\textbf{Short-term Forecasting Results}}
We evaluate short-term forecasting on four PEMS traffic benchmarks, where accurate prediction is challenging due to intricate spatiotemporal dependencies among sensors and strong inter-series interactions that drive citywide traffic dynamics.
Table~\ref{tab:short_term_results} summarizes the results under a fixed lookback length $T=96$ and prediction horizons $S\in\{12,24,48\}$.
Overall, \ourname{} achieves the best performance in all 16 settings (both MSE and MAE), demonstrating its strong effectiveness for short-term forecasting.
In terms of averaged MSE, \ourname{} improves over the strongest competitor by a clear margin, reducing errors by 25.8\% on PEMS03, 15.7\% on PEMS04, 12.7\% on PEMS07, and 16.4\% on PEMS08.
Notably, variate-independent architectures exhibit pronounced performance degradation on these datasets.
For example, PatchTST, which largely models each variate independently, is consistently inferior to \ourname{} across all horizons; a similar trend is observed for MLP-based methods, whose channel-independent mixing fails to capture dynamic cross-sensor interactions.
We attribute the consistent gains of \ourname{} to its explicit and effective cross-variate modeling: Mamba-DALA strengthens cross-variate interactions, including delay-aware dependencies among traffic sensors, and complements the cross-time temporal modeling of Mamba-SSD.
Their synergy enables \ourname{} to jointly exploit temporal dynamics and delay-aware inter-sensor dependencies, resulting in substantially improved short-term forecasting performance on the PEMS benchmarks.

\begin{table*}[!h]
  \caption{Imputation results on 96-length input windows with random masking ratios $\{12.5\%,25\%,37.5\%,50\%\}$. Lower MSE or MAE is better.}
  \label{tab:imputation_results}
  \renewcommand{\arraystretch}{0.85} 
  \centering
  \resizebox{\textwidth}{!}{
  \begin{threeparttable}
  \begin{small}
  \renewcommand{\multirowsetup}{\centering}
  \setlength{\tabcolsep}{1.45pt}
\begin{tabular}{cccccccccccccccccccccccccccc}
    \toprule
    \multicolumn{2}{c}{\multirow{3}{*}{Models}} &
    \multicolumn{2}{c}{\scalebox{0.9}{\textbf{SSD + DALA}}} &
    \multicolumn{8}{c}{\scalebox{0.8}{SSM-based}} &
    \multicolumn{6}{c}{\scalebox{0.8}{Transformer-based models}} &
    \multicolumn{6}{c}{\scalebox{0.8}{MLP-based models}} &
    \multicolumn{4}{c}{\scalebox{0.8}{TCN-based models}} \\ 
    \cmidrule(lr){3-4} \cmidrule(lr){5-12} \cmidrule(lr){13-18}\cmidrule(lr){19-24}\cmidrule(lr){25-28}

    \multicolumn{2}{c}{} & 
    \multicolumn{2}{c}{\rotatebox{0}{\scalebox{0.9}{\textbf{\ourname{}}}}}&
    
    \multicolumn{2}{c}{\rotatebox{0}{\scalebox{0.8}{Affirm}}} &
    \multicolumn{2}{c}{\rotatebox{0}{\scalebox{0.8}{S-Mamba}}} &
    \multicolumn{2}{c}{\rotatebox{0}{\scalebox{0.8}{CMamba}}} &
    \multicolumn{2}{c}{\rotatebox{0}{\scalebox{0.8}{SAMBA}}} &

    % \multicolumn{2}{c}{\rotatebox{0}{\scalebox{0.8}{FEDformer}}} &
    \multicolumn{2}{c}{\rotatebox{0}{\scalebox{0.8}{iTransformer}}} &
    \multicolumn{2}{c}{\rotatebox{0}{\scalebox{0.8}{Crossformer}}} &
    \multicolumn{2}{c}{\rotatebox{0}{\scalebox{0.8}{PatchTST}}} &

    % \multicolumn{2}{c}{\rotatebox{0}{\scalebox{0.8}{TiDE}}} &
    \multicolumn{2}{c}{\rotatebox{0}{\scalebox{0.8}{TimeMixer}}} &
    \multicolumn{2}{c}{\rotatebox{0}{\scalebox{0.8}{{RLinear}}}} &
    \multicolumn{2}{c}{\rotatebox{0}{\scalebox{0.8}{DLinear}}} &

    % \multicolumn{2}{c}{\rotatebox{0}{\scalebox{0.8}{SCINet}}} &
    \multicolumn{2}{c}{\rotatebox{0}{\scalebox{0.8}{ModernTCN}}} &
    \multicolumn{2}{c}{\rotatebox{0}{\scalebox{0.8}{TimesNet}}}
    \\
\bottomrule
  \end{tabular}
  \end{small}
  \end{threeparttable}
}
\end{table*}

\subsubsection{\textbf{Imputation Results}}
Time series imputation aims to recover missing values from partially observed series, which is crucial in real-world scenarios where data acquisition is imperfect due to sensor malfunctions or unexpected transmission glitches.
As reported in Table~\ref{tab:imputation_results}, \ourname{} consistently achieves the best performance across all datasets and masking ratios, surpassing both Transformer-based and SSM-based competitors by clear margins.
A notable trend is that the advantage of \ourname{} becomes more pronounced as the missing rate increases, where accurate recovery increasingly relies on fine-grained modeling of local fluctuations and precise temporal alignment between observed contexts and missing entries.
For instance, on ETTh1 with $50\%$ masking, \ourname{} attains a low error of $0.064/0.159$, whereas Transformer baselines (e.g., PatchTST: $0.173/0.271$; iTransformer: $0.142/0.286$) and SSM-based methods (e.g., S-Mamba: $0.142/0.281$) degrade substantially, indicating limited capability in capturing the local variations required for reliable completion.
These results suggest that imputation is not merely a global dependency modeling problem, but a fine-grained completion task that hinges on (i) local context continuity around missing positions and (ii) accurate relative lag relations between correlated variates.
Concretely, Mamba-SSD provides efficient temporal context aggregation to supply informative local cues. At the same time, Mamba-DALA enables delay-aware cross-variate interactions, enabling observed variates to contribute aligned evidence for reconstructing missing values.
Together, their synergy enables more precise recovery of local dynamics and inter-series signals, resulting in consistently superior imputation accuracy.

\subsection{\textbf{Classification Results}}
Time series classification aims to assign a semantic label to a multivariate series by recognizing both global patterns and discriminative local cues.
We evaluate on 10 multivariate datasets from the UEA Time Series Archive~\cite{bagnall2018uea} and report the average classification accuracy of each method across the 10 benchmarks.
As shown in Fig.\ref{fig:classification-anomaly}(a), \ourname{} achieves the highest average accuracy of $76.3\%$, consistently outperforming strong baselines from all major families, including ModernTCN ($74.2\%$), TimesNet ($73.6\%$), Crossformer ($73.2\%$), and TimeMixer ($73.1\%$).
A key observation is that TCN-based methods achieve competitive classification performance.
This is because temporal convolutions provide a strong inductive bias for extracting discriminative local motifs, e.g., shape-like patterns, that are often sufficient to determine class labels.
With hierarchical convolutions, TCNs can further aggregate multi-scale evidence while remaining robust to small phase shifts and noise through parameter sharing and localized receptive fields, making them particularly effective for semantic recognition.
In contrast, MLP-based approaches remain inferior (e.g., DLinear: $67.5\%$, RLinear: $70.0\%$), since their predominantly linear mixing is less capable of building high-level semantic features and modeling complex inter-series interactions.
The superior performance of \ourname{} stems from its fine-grained and hierarchical dependency encoding: stacking DuoMNet blocks progressively aggregates multi-scale temporal evidence, while the dual-path design explicitly captures both temporal variations and cross-variate cues.
In particular, Mamba-SSD strengthens long-range temporal structure modeling, and Mamba-DALA enhances cross-variate interactions, enabling \ourname{} to form more informative and semantically aligned representations for classification beyond purely convolutional local feature extraction.

\begin{figure*}[h]
    \centering
    \begin{subfigure}[b]{0.49\textwidth}
        \centering
        \includegraphics[width=\linewidth]{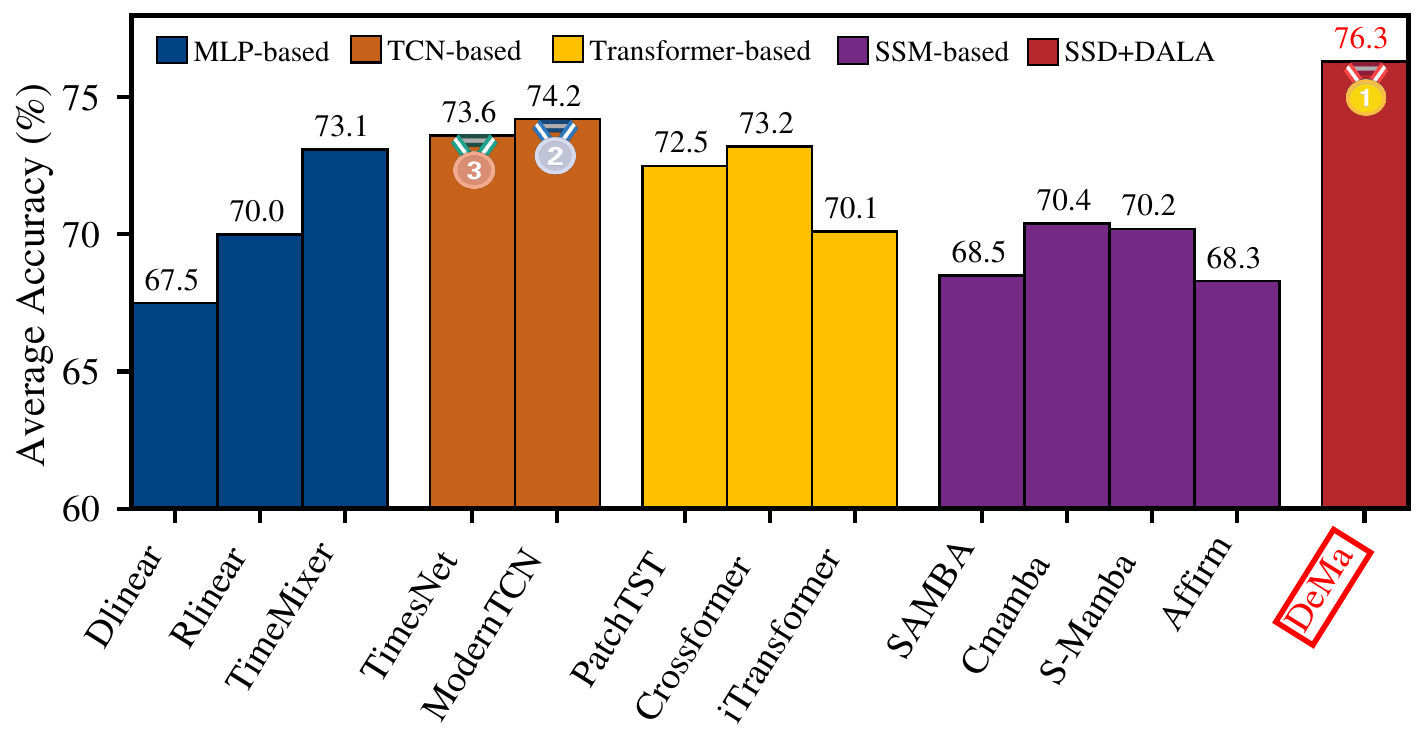}
        \caption{Classification results (average accuracy).}
        \label{fig:classification_result}
    \end{subfigure}
    \hfill
    \begin{subfigure}[b]{0.49\textwidth}
        \centering
        \includegraphics[width=\linewidth]{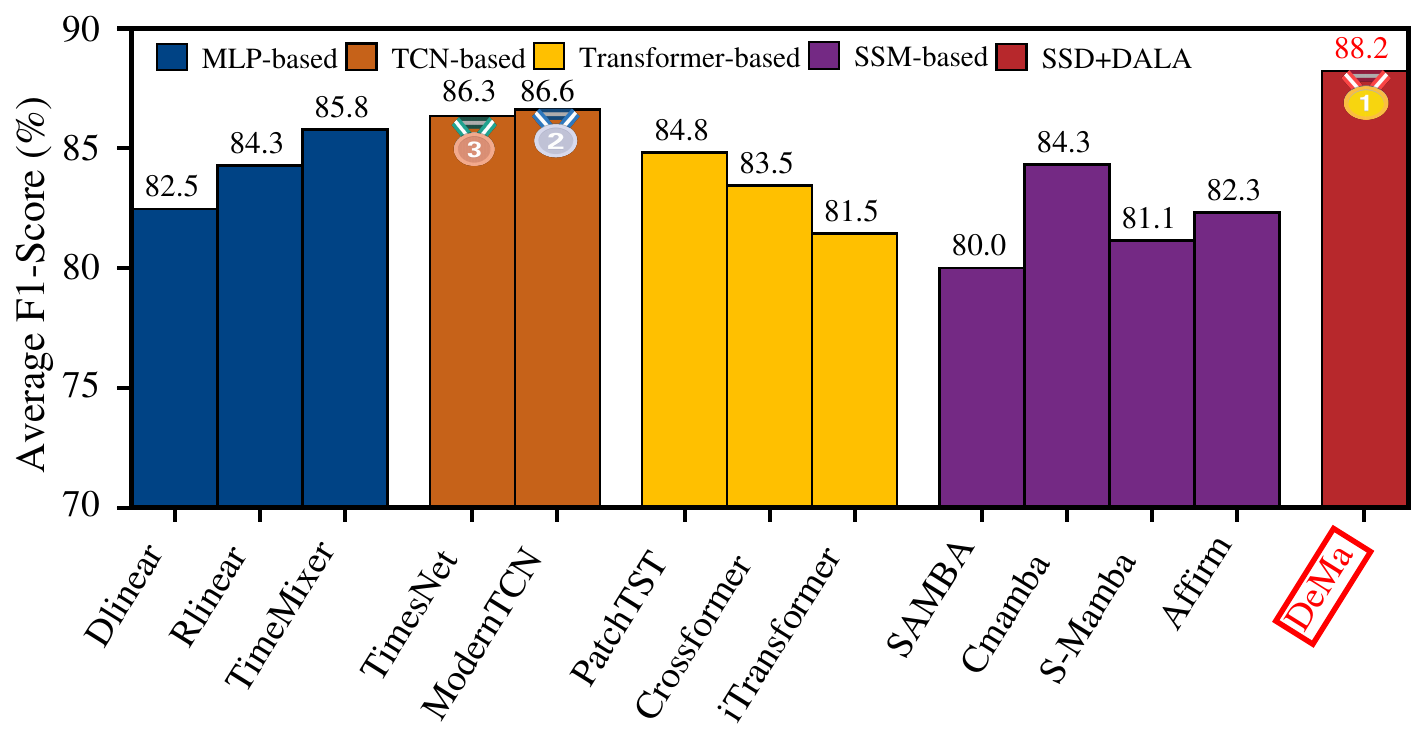}
        \caption{Anomaly detection results (average F1-score).}
        \label{fig:anomaly_result}
    \end{subfigure}
    \caption{Performance comparison on classification and anomaly detection tasks. Results are averaged over multiple datasets, and higher values indicate better performance.}
    \label{fig:classification-anomaly}
\end{figure*}

\subsubsection{\textbf{Anomaly Detection Results}}
Anomaly detection aims to identify rare and abnormal patterns in time series, which often correspond to faults, critical events, or outliers requiring timely intervention.
Following prior work~\cite{wang2024deep}, we evaluate on five widely used anomaly detection benchmarks and report the average F1-score across datasets~\cite{tu2026generalized}.
For fair comparison, we adopt reconstruction error~\cite{Timesnet} as the anomaly criterion for all methods.
As shown in Fig.\ref{fig:classification-anomaly}(b), \ourname{} achieves the best average F1-score of $88.2\%$, outperforming strong competitors such as ModernTCN ($86.6\%$) and TimesNet ($86.3\%$).
This gain highlights the advantage of our fine-grained modeling.
The proposed SSD$+$DALA design jointly strengthens (i) local temporal continuity and long-range temporal context via Mamba-SSD, and (ii) delay-aware cross-variate interactions via Mamba-DALA, enabling more faithful reconstruction of normal dynamics.
As a result, abnormal deviations yield sharper and more separable reconstruction residuals, leading to improved detection accuracy.
In contrast, iTransformer~\cite{liuitransformer} and S-Mamba~\cite{wang2024mamba} perform notably worse, likely because prevalent normal patterns can dominate their series-wise global similarity modeling by attention; consequently, the subtle abnormal segments are more easily diluted when aggregating pairwise correlations, leading to inferior detection.
Finally, we observe that methods benefiting from explicit decomposition cues (e.g., TimesNet and \ourname{}) tend to achieve stronger overall detection performance, underscoring the importance of separating stable periodic structures from irregular fluctuations so that violations of normal regularities become more detectable.
\begin{figure*}[h]
    \centering
    \begin{subfigure}[b]{0.53\textwidth}
        \centering
        \includegraphics[width=\linewidth]{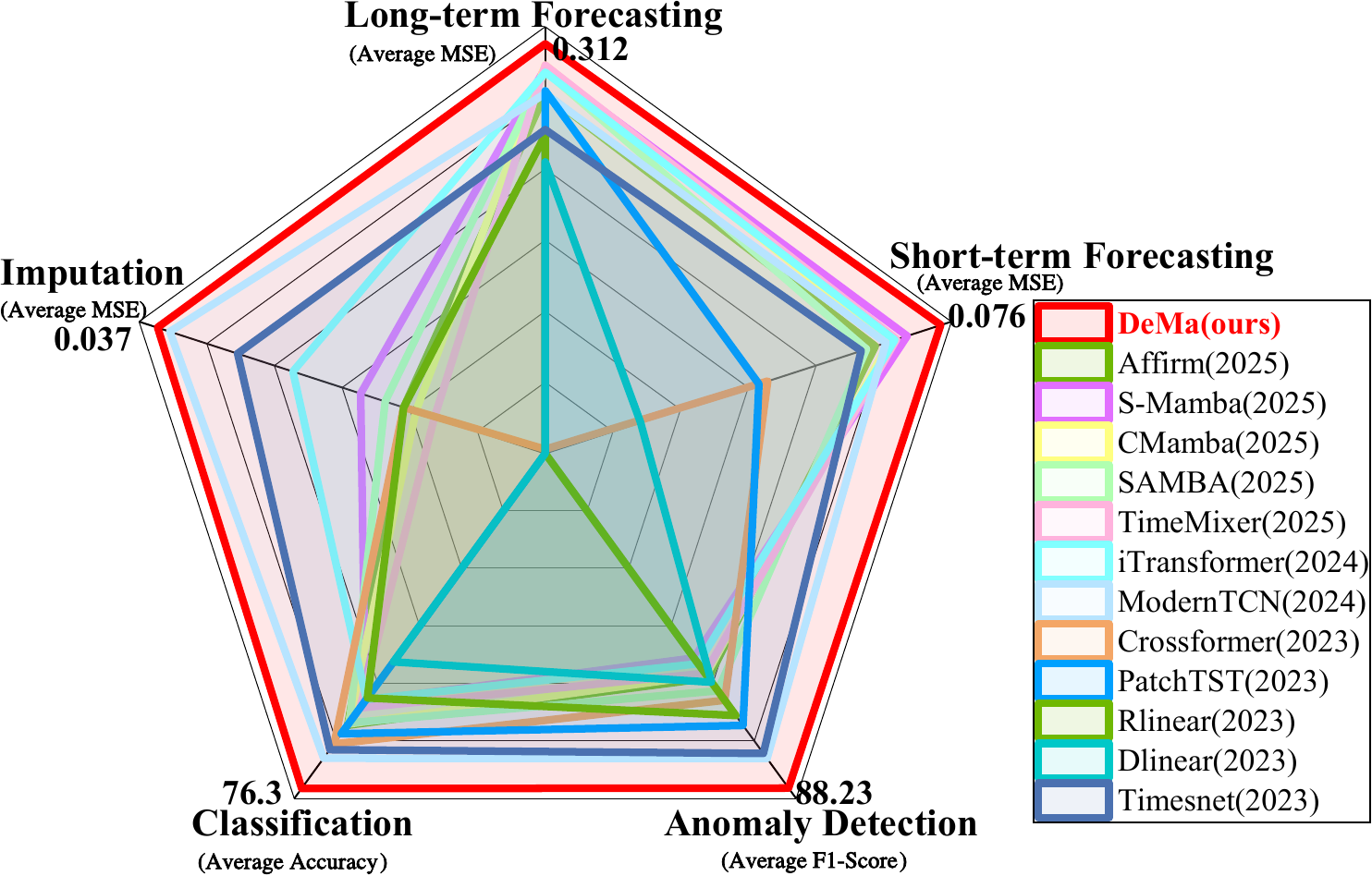}
        \caption{Overall performance across five tasks.}
        \label{fig:five_tasks}
    \end{subfigure}
    \hfill
    \begin{subfigure}[b]{0.43\textwidth}
        \centering
        \includegraphics[width=\linewidth]{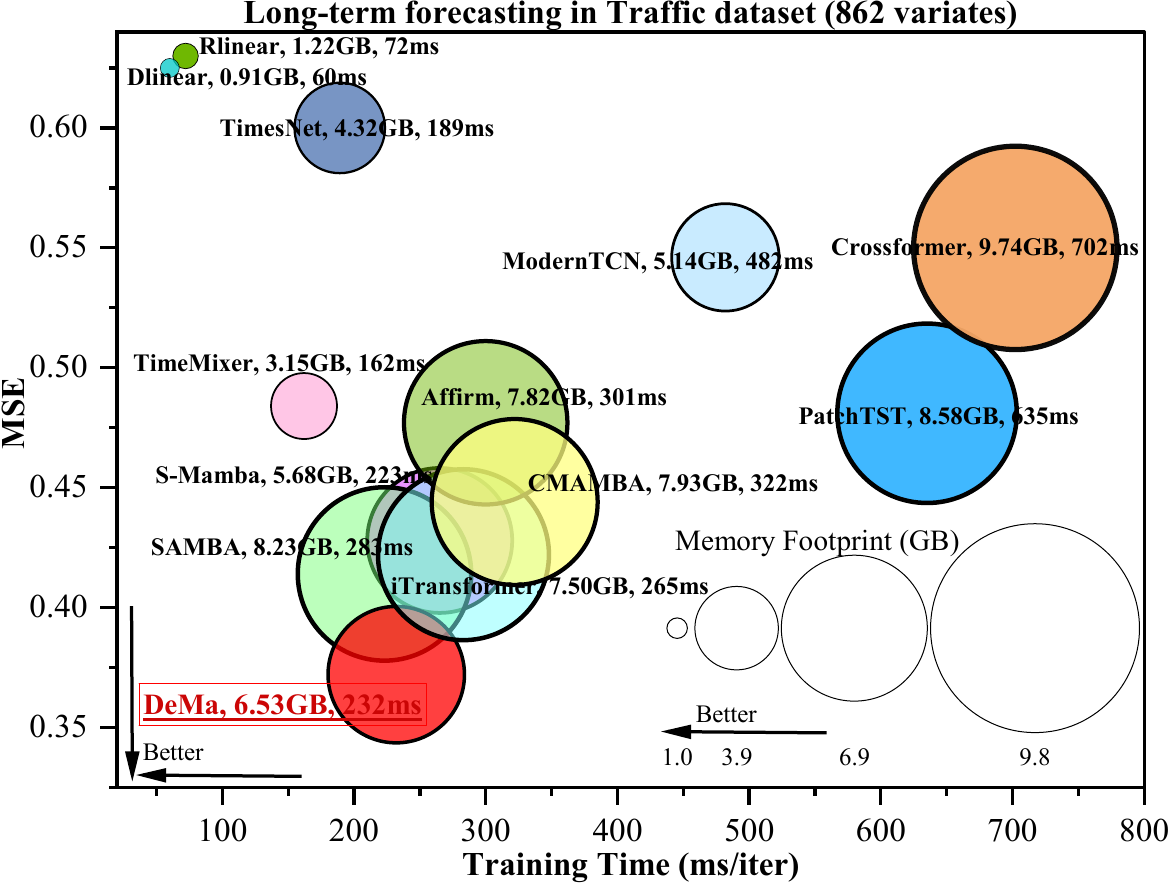}
        \caption{Performance--efficiency trade-off.}
        \label{fig:efficiency_tradeoff}
    \end{subfigure}
    \caption{Comparison of model performance and efficiency. \ourname{} achieves consistent state-of-the-art results across five mainstream time-series tasks. It also attains the lowest mean squared error while requiring less training time and GPU memory.}
    \label{fig:overall}
\end{figure*}
\subsection{\textbf{Performance--Efficiency Trade-off (RQ2)}}
Fig.\ref{fig:overall} summarizes the overall effectiveness and efficiency of \ourname{}.
Fig.\ref{fig:overall}(a) provides an overall comparison of baselines across five representative time-series tasks.
Overall, \ourname{} achieves consistently strong results on all tasks, demonstrating robust task generality rather than tuning to a single forecasting setting.
Fig.\ref{fig:overall}(b) further evaluates the accuracy-efficiency trade-off on a large-scale traffic dataset with $862$ variates under long-term forecasting, where the y-axis denotes MSE, the x-axis indicates training time, and the bubble size reflects GPU memory footprint.
As shown, \ourname{} attains the lowest MSE ($0.372$) with moderate training time ($232$ ms/iter) and manageable memory usage ($6.53$~GB), offering a favorable balance between accuracy and computational cost.
These results suggest that \ourname{} is a practical and scalable solution for large-scale multivariate time-series analysis.

\ourname{} exhibits clear efficiency advantages over Transformer-based baselines while remaining more accurate, benefiting from the linear-time state-space backbone.
For instance, Crossformer and PatchTST require substantially higher training time and memory footprint (e.g., $702$ ms/iter and $9.74$~GB for Crossformer), yet still yield higher MSE than \ourname{} in Fig.\ref{fig:overall}(b).
Moreover, \ourname{} consistently outperforms TCN-based models in the multi-task comparison (Fig.\ref{fig:overall}(a)).
While TCNs are efficient and effective at capturing local motifs, their limited receptive field and forecasting-oriented inductive bias may limit transferability to tasks that require long-range aggregation.
In contrast, \ourname{} learns more transferable representations by jointly encoding long-range temporal dynamics and delay-aware inter-series interactions.
Finally, \ourname{} improves the trade-off between accuracy and efficiency compared to recent Mamba variants.
As shown in Fig.\ref{fig:overall}(b), \ourname{} is faster and lighter than Affirm and CMamba (e.g., $232$ ms/iter and $6.53$~GB for \ourname{} versus $301$ ms/iter and $7.82$~GB for Affirm, and $322$ ms/iter and $7.93$~GB for CMamba), while achieving lower MSE.

In summary, \ourname{} combines state-of-the-art accuracy with practical training speed and memory efficiency, making it a strong candidate for scalable, general time-series analysis.

\subsection{\textbf{Scalability vs. Input Length (RQ3)}}
Semantic information in time series is often formed by aggregating evidence over long horizons rather than isolated timestamps. Although longer lookback windows provide richer context, they also place stricter demands on training efficiency and GPU memory. To evaluate scalability with respect to the input length, Fig.~\ref{fig:GPU} reports the per-iteration running time and GPU memory footprint of \ourname{} and representative Transformer-based baselines on ETTh1. We increase the lookback length from $\{384,768,1536,3072\}$ while fixing the embedding dimension to $d=256$ for a fair comparison. 

As shown, \ourname{} consistently achieves the lowest running time and memory usage across all lengths, and its growth remains close to linear as the series length increases. By contrast, Transformer-based methods (e.g., Transformer, PatchTST, and Crossformer) exhibit rapidly increasing computation and memory overhead with longer inputs, which substantially limits their practicality in long-term settings. Notably, iTransformer is more efficient than token-wise Transformers because it computes series-wise attention, avoiding quadratic growth with respect to the token length; however, it remains consistently slower and more memory-intensive than \ourname{} at large input lengths. These results further explain why the advantages of DeMa become especially pronounced when the input context grows: the dual path preserves near-linear scaling with series length, whereas attention-based baselines incur substantially steeper computational and memory growth.

\begin{figure*}[h]
\centering
\includegraphics[width=0.6\textwidth]{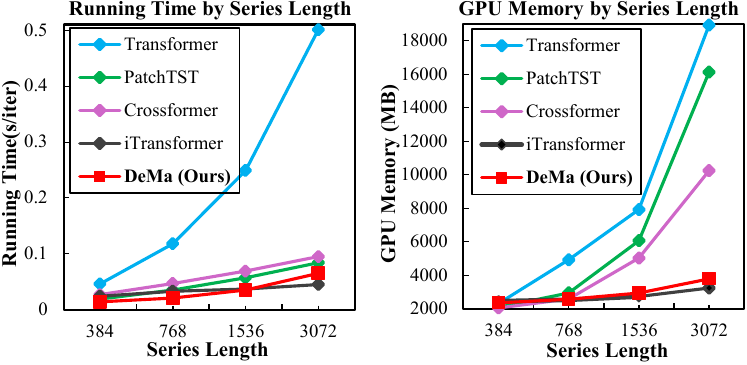}
\caption{Efficiency analysis of GPU memory and running time in a long-term lookback-window scenario. Our proposed \ourname{} scales linearly with the series length, whereas the vanilla Transformer exhibits quadratic complexity with respect to the series length.}
\label{fig:GPU}
\end{figure*}

These observations align with our complexity analysis in Fig.\ref{fig:motivation}(b). Leveraging the parallelizable Mamba-SSD and Mamba-DALA blocks, \ourname{} scales linearly with the token length, with time complexity $\mathcal{O}(2NLD^2)$, where $N$ is the number of variates, $L$ is the number of tokens after tokenization, and $D$ is the embedding dimension. In typical long-term MTS scenarios, $L, N \gg D$, making \ourname{} a more scalable backbone for long-term and large-scale multivariate time-series modeling.

\subsection{\textbf{Ablation Study (RQ4)}}
To quantify the contribution of each key component in \ourname{}, including the Adaptive Fourier Filter, \SSD{}, and \DALA{}, we conduct ablation studies on Traffic (long-term forecasting) and PEMS04 (short-term forecasting). 
As reported in Table~\ref{tab:ablation}, we consider two types of ablations: \emph{replacement} (\emph{Repl.}), which substitutes a target module with an alternative design, and \emph{removal} (\emph{w/o}), which disables the module entirely. 
Overall, each component consistently improves forecasting accuracy, and the full model yields the best performance on both datasets.
\begin{table*}[!h]
\caption{Ablation study of \ourname{}.}
\label{tab:ablation}

\begin{small}
\renewcommand{\multirowsetup}{\centering}
\renewcommand{\arraystretch}{1}

\begin{tabular*}{\textwidth}{@{\extracolsep{\fill}}clccccccc@{}}
\toprule
\multirow{2}{*}{No.} & \multirow{2}{*}{Design} & \multirow{2}{*}{Decomp.} & \multirow{2}{*}{Variate Module} & \multirow{2}{*}{Temporal Module} & \multicolumn{2}{c}{Traffic} & \multicolumn{2}{c}{PEMS04} \\
\cmidrule(lr){6-7} \cmidrule(lr){8-9}
& & & & & MSE & MAE & MSE & MAE \\
\midrule
1 & \textbf{\ourname{}} & \textbf{AFF} & \textbf{\DALA{}} & \textbf{\SSD{}} & \textbf{0.382} & \textbf{0.247} & \textbf{0.061} & \textbf{0.168} \\
\midrule
2 & \multirow{3}{*}{Repl.} & AFF & \SSD{} & \DALA{} & 0.416 & 0.297 & 0.093 & 0.204 \\
3 &                        & AFF & \DALA{} & \DALA{} & 0.507 & 0.314 & 0.126 & 0.257 \\
4 &                        & AFF & \SSD{}  & \SSD{}  & 0.412 & 0.295 & 0.087 & 0.192 \\
\midrule
5 & \multirow{3}{*}{w/o}   & w/o & \DALA{} & \SSD{}  & 0.415 & 0.279 & 0.097 & 0.215 \\
6 &                        & w/o & w/o     & \SSD{}  & 0.539 & 0.317 & 0.142 & 0.267 \\
7 &                        & w/o & \DALA{} & w/o     & 0.597 & 0.383 & 0.225 & 0.337 \\
\bottomrule
\end{tabular*}
\end{small}
\end{table*}

\begin{itemize}[leftmargin=*]
\item \textbf{Adaptive Fourier Filter.}
Comparing Rows~1 and~5, removing the decomposition module causes clear degradation on both datasets (Traffic MSE: $0.382\!\rightarrow\!0.415$; PEMS04 MSE: $0.061\!\rightarrow\!0.097$). 
This verifies that the Adaptive Fourier Filter provides a beneficial decomposition prior by separating slowly varying structures (e.g., trend and seasonality) from short-term fluctuations, thereby reducing spectral interference and facilitating more reliable dependency modeling in both the temporal and variate paths, consistent with decomposition-based forecasting literature~\cite{liu2023koopa, wu2021autoformer}.

\item \textbf{\SSD{} for temporal modeling.}
Rows~2--3 replace the temporal \SSD{} with attention-style alternatives, leading to substantial performance drops (e.g., Row~3: Traffic MSE $0.382\!\rightarrow\!0.507$, a $32.7\%$ increase). 
This indicates that \SSD{} is critical for efficiently capturing long-range temporal dynamics. 
Moreover, removing the temporal path entirely (Row~7) results in the largest degradation (Traffic MSE $0.382\!\rightarrow\!0.597$, $+56.3\%$; PEMS04 MSE $0.061\!\rightarrow\!0.225$), confirming that strong cross-time modeling is indispensable for accurate forecasting.

\item \textbf{\DALA{} for cross-variate interaction modeling.}
Rows~4 and~6 highlight the necessity of \DALA{} for modeling cross-variate dependencies. 
Replacing \DALA{} with a temporal-style \SSD{} interaction (Row~4) already hurts performance (Traffic MSE $0.382\!\rightarrow\!0.412$; PEMS04 MSE $0.061\!\rightarrow\!0.087$), suggesting that directly reusing temporal mechanisms is insufficient for variate interactions.
More importantly, removing \DALA{} (Row~6) causes a pronounced drop (Traffic MSE $0.382\!\rightarrow\!0.539$, $+41.1\%$; PEMS04 MSE $0.061\!\rightarrow\!0.142$), verifying that delay-aware cross-variate modeling provides critical complementary cues beyond temporal dynamics.
\end{itemize}

Overall, the ablation results show that the main components of DeMa are functionally complementary rather than redundant. This supports that the structured design of DeMa is challenge-driven and empirically justified, rather than an unnecessarily complicated composition.

\begin{figure*}[!h]
    \centering

    % Row 1
    \begin{minipage}{0.45\textwidth}
        \centering
        \subfloat[Weather Forecasting: MAE vs. $(\alpha,\beta)$.]{
            \includegraphics[width=\linewidth]{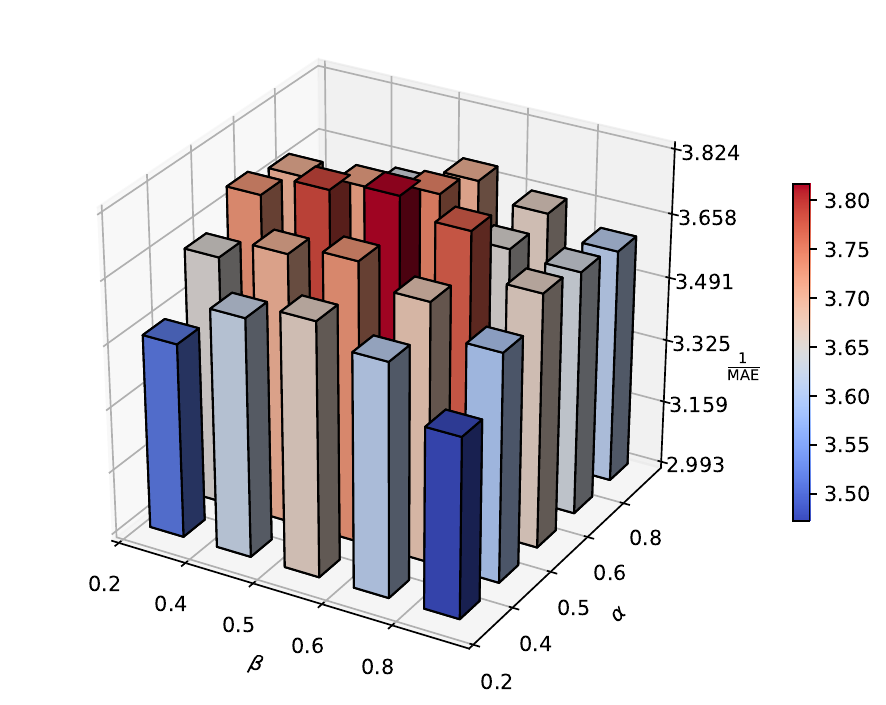}
        }
    \end{minipage}
    \begin{minipage}{0.45\textwidth}
        \centering
        \subfloat[Heartbeat Classification: Accuracy vs. $(\alpha,\beta)$.]{
            \includegraphics[width=\linewidth]{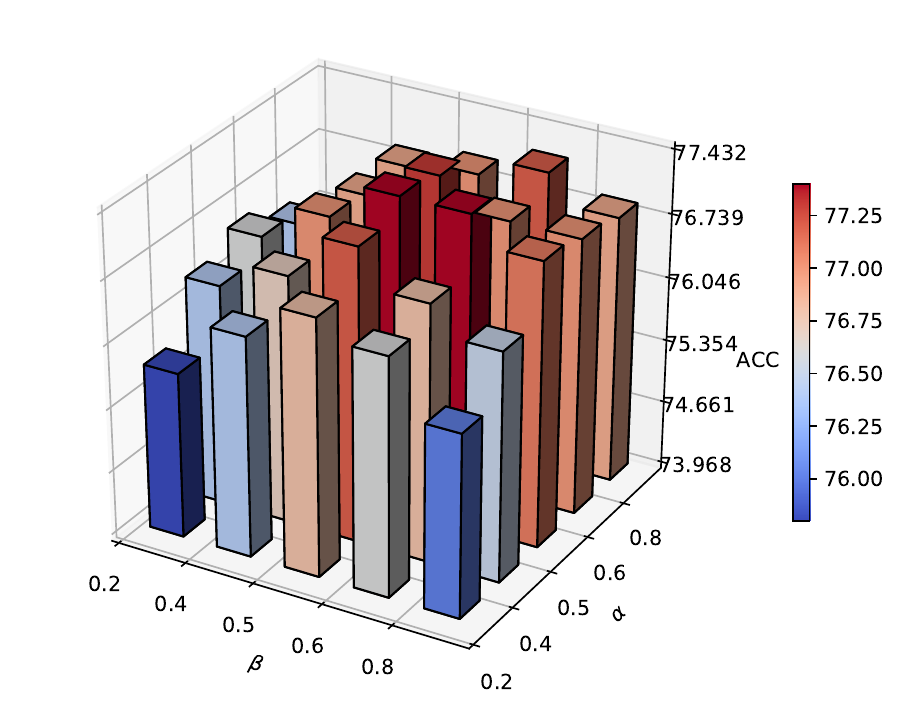}
        }
    \end{minipage}

    \vspace{0.1em}

    % Row 2
    \begin{minipage}{0.45\textwidth}
        \centering
        \subfloat[Weather Imputation: MAE vs. $(\alpha,\beta)$.]{
            \includegraphics[width=\linewidth]{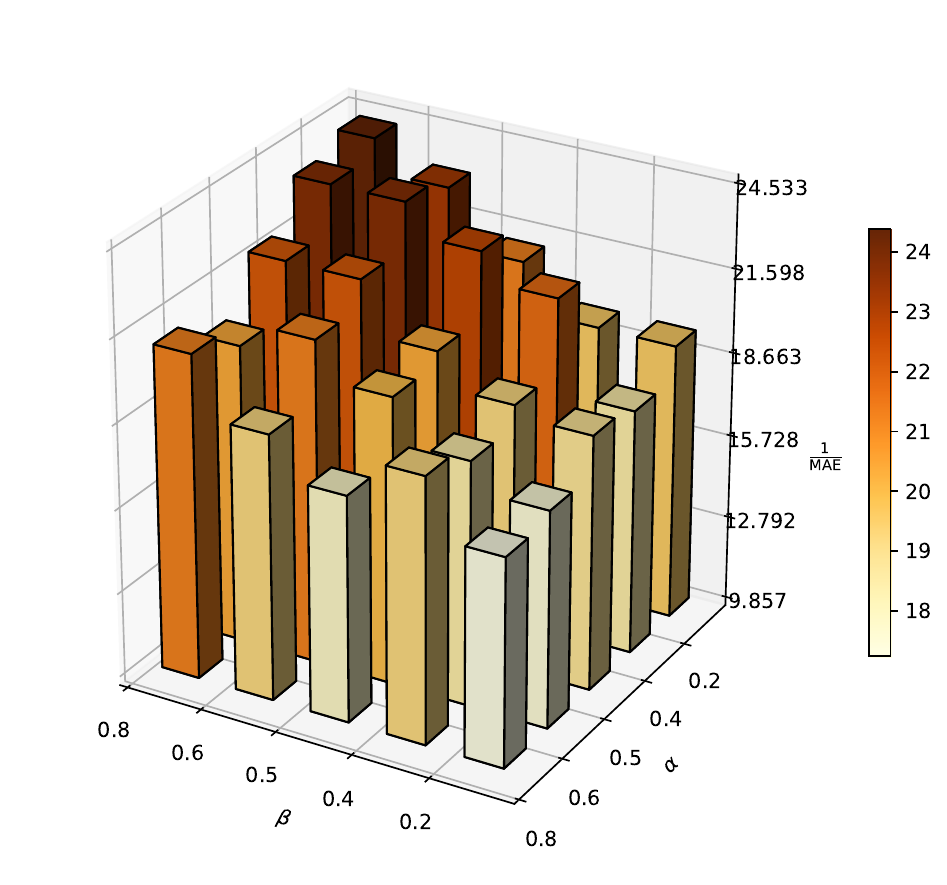}
        }
    \end{minipage}
    \begin{minipage}{0.45\textwidth}
        \centering
        \subfloat[MSL Anomaly Detection: F1-Score vs. $(\alpha,\beta)$.]{
            \includegraphics[width=\linewidth]{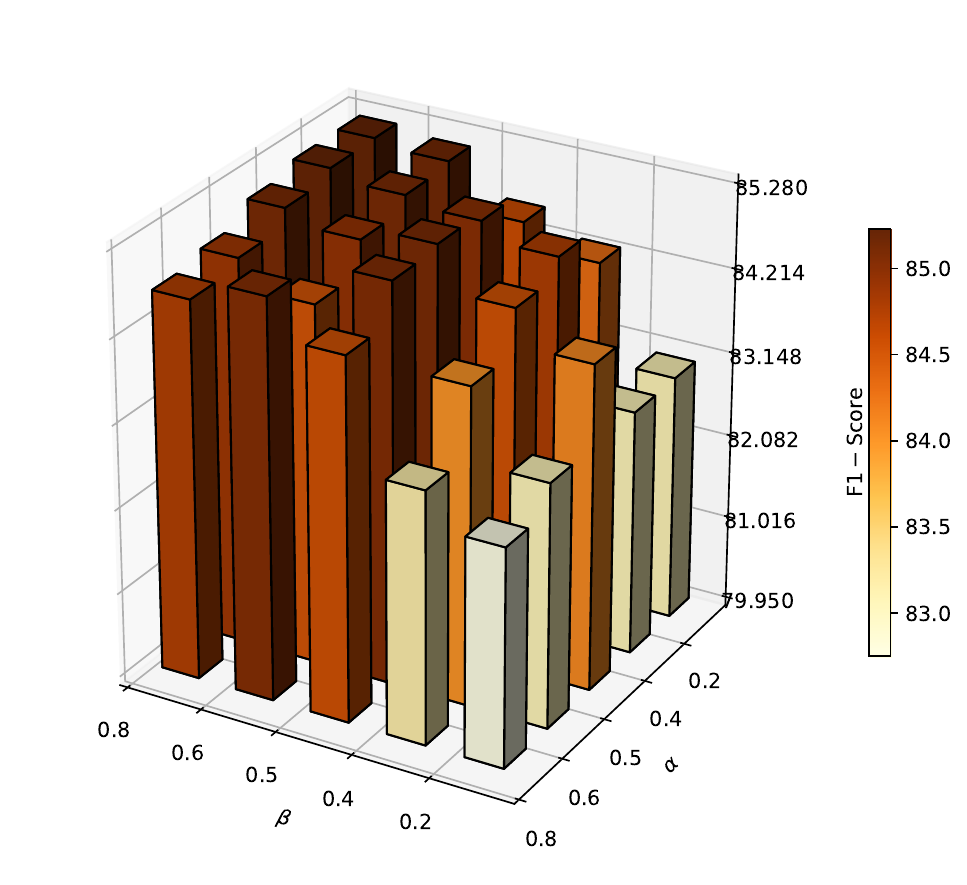}
        }
    \end{minipage}

    \caption{Sensitivity of Fusion Weights $\alpha$ and $\beta$ Across Tasks.}
\label{fig:hyperparameter_analysis}
\end{figure*}

\subsection{\textbf{Hyperparameter Sensitivity Analysis (RQ5)}}
In \ourname{}, we adopt a weighted fusion layer to combine the temporal-path representation and the variate-path representation, controlled by the fusion weights $\alpha$ and $\beta$ in Eq.~\eqref{eq:duomnet_mix}.
To examine the sensitivity of this fusion trade-off, we perform a grid search over $\alpha,\beta\in\{0.2,0.4,0.5,0.6,0.8\}$ on forecasting and imputation on Weather~\cite{wu2021autoformer}, classification on Heartbeat~\cite{bagnall2018uea}, and anomaly detection on MSL~\cite{hundman2018detecting}.
The results are summarized in Fig.\ref{fig:hyperparameter_analysis}(a)-(d), where higher is better for Accuracy, F1-score, and $1/\text{MAE}$.

Across all tasks, \ourname{} is most reliable when both paths remain active (i.e., neither $\alpha$ nor $\beta$ is overly small), confirming that temporal modeling and cross-variate interaction are complementary rather than substitutable.
Forecasting is relatively less sensitive to $\alpha$ and $\beta$ (Fig.\ref{fig:hyperparameter_analysis}(a)), showing only mild variation across the grid.
An explanation is that long-horizon prediction is often dominated by stable structures such as trend and seasonality; thus, once the temporal path is sufficiently emphasized, the model can still perform well even under moderate fusion bias, and the variate path provides a relatively smaller marginal gain unless the dataset is strongly coupled.
Similarly, classification remains stable across a wide range of $(\alpha,\beta)$ (Fig.\ref{fig:hyperparameter_analysis}(b)).
This is because classification primarily relies on global semantic discrimination: the stacked DuoMNet blocks can progressively aggregate high-level evidence and compensate for moderate fusion imbalance, as long as neither path is completely suppressed and both temporal and cross-variate cues are retained.

In contrast, imputation is more sensitive to the fusion weights (Fig.\ref{fig:hyperparameter_analysis}(c)), because recovering missing points typically requires both local temporal continuity (temporal path) and cross-variate complementarity (variate path). Over-weighting either path can lead to incomplete recovery, e.g., temporal-only fusion may miss informative cross-variate cues, while variate-only fusion may fail to preserve fine-grained local dynamics, especially under higher missing ratios. 
Similarly, anomaly detection exhibits strong sensitivity (Fig.\ref{fig:hyperparameter_analysis}(d)), since anomalies are rare, local patterns that demand temporal modeling to capture abrupt deviations, while cross-variate modeling helps validate abnormality via multi-sensor consistency; improper weighting may either dilute anomalies under dominant normal patterns or amplify spurious fluctuations, increasing false alarms. These observations suggest that $\alpha$ and $\beta$ should be selected task-awarely, with balanced fusion particularly important for fine-grained, locality-critical tasks such as imputation and anomaly detection.

\section{RELATED WORK}
\label{sec:literature}

\subsection{\textbf{Transformer-based Time Series Models}}
Transformers have become a dominant paradigm for multivariate time series (MTS) modeling due to their strong ability to capture pairwise dependencies with a global receptive field.
Representative approaches include Informer~\cite{Informer}, Fedformer~\cite{fedformer}, Autoformer~\cite{wu2021autoformer}, Pyraformer~\cite{Pyraformer}, PatchTST~\cite{PatchTST}, Crossformer~\cite{Crossformer}, and UniTime~\cite{liu2024unitime}.
A fundamental limitation of these models lies in the quadratic complexity of self-attention with respect to the token length $L$, which hinders scalability to long-horizon settings.
To alleviate this issue, sparse or structured attention has been explored to reduce the computational burden (e.g., to $\mathcal{O}(L\log L)$)~\cite{kitaev2020reformer, Informer, li2019enhancing}. Yet, it may overlook informative dependencies when only a subset of interactions is retained.
More recently, iTransformer~\cite{liuitransformer} reformulates MTS tokens by treating each variate as a token and applying attention over variates, shifting the quadratic term from $L$ to the number of variates $N$.
While efficient when $N \ll L$, this design may sacrifice fine-grained local temporal interactions that are crucial for locality-sensitive tasks.

\subsection{\textbf{MLP-based Time Series Models}}
MLP-style architectures offer an attractive alternative for efficient time-series modeling by replacing attention with lightweight mixing operations.
Classic representatives include DLinear~\cite{DLinear} and its variants such as RLinear~\cite{li2023revisiting}, which adopt decomposition-aware linear mapping strategies to achieve strong forecasting performance with low computational cost.
TimeMixer~\cite{timemixer} further enhances expressiveness via MLP-based mixing across temporal and channel dimensions.
Despite their efficiency, purely MLP-based designs often rely on relatively limited inductive biases for explicitly modeling complex cross-variate interactions and fine-grained dependency structures, which can restrict their generality across diverse MTS tasks beyond forecasting.

\subsection{\textbf{TCN-based Time Series Models}}
Temporal convolutional networks (TCNs) model temporal dependencies via convolutional receptive fields and have been widely used in time series forecasting and representation learning.
ModernTCN~\cite{luo2024moderntcn} strengthens classical TCNs with modern architectural practices to improve both accuracy and efficiency, while TimesNet~\cite{Timesnet} leverages convolutional modeling together with period-aware representations to capture multi-period temporal patterns.
Although convolution provides strong locality and parallelism, TCN-based models typically require carefully designed receptive fields to capture long-range dependencies, and explicitly modeling cross-variate interactions remains non-trivial when scaling to high-dimensional MTS.

\subsection{\textbf{Mamba-based Time Series Models}}
State space models (SSMs), particularly Mamba-style selective SSMs, have recently emerged as a compelling linear-time alternative to Transformers, demonstrating strong performance across diverse domains~\cite{qu2024survey}.
This progress has spurred increasing interest in adapting Mamba to time series analysis~\cite{cmamba, wang2024mamba, behrouz2024mambamixer, cai2024mambats, wu2025affirm, samba}.
For example, CMamba~\cite{cmamba} augments Mamba-style temporal modeling with extra modules to capture cross-variate interactions, while S-Mamba~\cite{wang2024mamba}, MambaMixer~\cite{behrouz2024mambamixer}, MambaTS~\cite{cai2024mambats}, and more recent variants such as Affirm~\cite{wu2025affirm} and SAMBA~\cite{samba} explore architectural adaptations to improve effectiveness on time series benchmarks.
Despite these advances, existing Mamba-based designs still face notable challenges for general MTS analysis.
First, dependency entanglement is common: temporal dynamics and cross-variate interactions frequently overlap and are mixed, which can obscure task-relevant cues.
Second, fine-grained cross-variate modeling remains insufficient: many methods primarily emphasize global interactions and may under-exploit local, delay-sensitive cross-variate effects.
Motivated by these gaps, we propose \textbf{\ourname{}}, which disentangles cross-time and cross-variate dependency encodings via a dual-path architecture, enabling efficient yet effective modeling for general time-series analysis.

\section{Conclusion}
\label{sec:conclusion}

In this work, we propose \ourname{}, a dual-path delay-aware Mamba backbone for efficient multivariate time-series analysis. \ourname{} explicitly decomposes the time-series context into \textit{Cross-Time} and \textit{Cross-Variate} components via an Adaptive Fourier Filter. These components are processed by stacked DuoMNet blocks with two parallel, scan-coupled pathways: (i) a temporal path that applies Cross-Time Scan and Mamba-SSD to capture long-range intra-series dependencies in a series-independent and parallel manner, and (ii) a variate path that applies Cross-Variate Scan and Mamba-DALA to model delay-aware inter-series interactions using DALA with both global correlation delays and token-level relative delays. The resulting representations are fused through a lightweight weighted fusion layer and projected to task-specific outputs. Extensive experiments across five mainstream tasks demonstrate that \ourname{} achieves consistently strong accuracy while substantially reducing training time and GPU memory usage, suggesting a favorable accuracy-efficiency trade-off and practical scalability for long-horizon and large-scale MTS modeling.

\section*{Acknowledgments}
The research described in this paper has been partially supported by the National Natural Science Foundation of China (project no. 62433016 and 62537001), the General Research Funds from the Hong Kong Research Grants Council (project No. PolyU 15207322, 15200023, 15206024, and 15224524), Hong Kong Research Grants Council’s Theme-based Research Scheme (No. T43-513/23-N), Hong Kong Research Grants Council’s Research Impact Fund (No. R1015-23), Hong Kong Research Grants Council’s Collaborative Research Fund (No. C1043-24GF), Internal research funds from Hong Kong Polytechnic University (project no. P0059586, P0042693, P0048625, and P0051361), and Sheertek International (HK) Limited. This work was supported by computational resources provided by The Centre for Large AI Models (CLAIM) of The Hong Kong Polytechnic University.
\small
\bibliographystyle{unsrt}
\bibliography{references/references}

@article{wang2024deep,
  author       = {Yuxuan Wang and Haixu Wu and Jiaxiang Dong and Yong Liu and Mingsheng Long and Jianmin Wang},
  title        = {Deep Time Series Models: {A} Comprehensive Survey and Benchmark},
  journal      = {arXiv preprint arXiv:2407.13278},
  year         = {2024}
}

@article{huo2023hierarchical,
  author       = {Guangyu Huo and Yong Zhang and Boyue Wang and Junbin Gao and Yongli Hu and Baocai Yin},
  title        = {Hierarchical Spatio-Temporal Graph Convolutional Networks and Transformer Network for Traffic Flow Forecasting},
  journal      = {IEEE Transactions on Intelligent Transportation Systems},
  volume       = {24},
  number       = {4},
  pages        = {3855--3867},
  year         = {2023},
  doi          = {10.1109/TITS.2023.3234512}
}

@article{qin2020imaging,
  author       = {Zhen Qin and Yibo Zhang and Shuyu Meng and Zhiguang Qin and Kim{-}Kwang Raymond Choo},
  title        = {Imaging and fusing time series for wearable sensor-based human activity recognition},
  journal      = {Information Fusion},
  volume       = {53},
  pages        = {80--87},
  year         = {2020},
  doi          = {10.1016/J.INFFUS.2019.06.014}
}

@article{blazquez2021review,
  author       = {Ane Bl{\'{a}}zquez{-}Garc{\'{\i}}a and Angel Conde and Usue Mori and Jos{\'{e}} Antonio Lozano},
  title        = {A Review on Outlier/Anomaly Detection in Time Series Data},
  journal      = {ACM Computing Surveys},
  volume       = {54},
  number       = {3},
  pages        = {56:1--56:33},
  year         = {2022},
  doi          = {10.1145/3444690}
}

@inproceedings{chen2024lara,
  author       = {Feiyi Chen and Zhen Qin and Mengchu Zhou and Yingying Zhang and Shuiguang Deng and Lunting Fan and Guansong Pang and Qingsong Wen},
  title        = {{LARA:} {A} Light and Anti-overfitting Retraining Approach for Unsupervised Time Series Anomaly Detection},
  booktitle    = {Proceedings of the ACM Web Conference (WWW)},
  pages        = {4138--4149},
  publisher    = {ACM},
  year         = {2024},
  doi          = {10.1145/3589334.3645472}
}

@article{hewamalage2021recurrent,
  author       = {Hansika Hewamalage and Christoph Bergmeir and Kasun Bandara},
  title        = {Recurrent Neural Networks for Time Series Forecasting: Current Status and Future Directions},
  journal      = {arXiv preprint arXiv:1909.00590},
  year         = {2019}
}

@article{SegRNN,
  author       = {Shengsheng Lin and Weiwei Lin and Wentai Wu and Feiyu Zhao and Ruichao Mo and Haotong Zhang},
  title        = {SegRNN: Segment Recurrent Neural Network for Long-Term Time-Series Forecasting},
  journal      = {IEEE Internet of Things Journal},
  volume       = {13},
  number       = {5},
  pages        = {9861--9871},
  year         = {2026},
  doi          = {10.1109/JIOT.2025.3647705}
}

@inproceedings{Timesnet,
  author       = {Haixu Wu and Tengge Hu and Yong Liu and Hang Zhou and Jianmin Wang and Mingsheng Long},
  title        = {{TimesNet}: Temporal 2D-Variation Modeling for General Time Series Analysis},
  booktitle    = {Proceedings of the International Conference on Learning Representations (ICLR)},
  publisher    = {OpenReview},
  year         = {2023}
}

@inproceedings{PatchTST,
  author       = {Yuqi Nie and Nam H. Nguyen and Phanwadee Sinthong and Jayant Kalagnanam},
  title        = {A Time Series is Worth 64 Words: Long-term Forecasting with Transformers},
  booktitle    = {Proceedings of the International Conference on Learning Representations (ICLR)},
  publisher    = {OpenReview},
  year         = {2023}
}

@inproceedings{Crossformer,
  author       = {Yunhao Zhang and Junchi Yan},
  title        = {Crossformer: Transformer Utilizing Cross-Dimension Dependency for Multivariate Time Series Forecasting},
  booktitle    = {Proceedings of the International Conference on Learning Representations (ICLR)},
  publisher    = {OpenReview},
  year         = {2023}
}

@inproceedings{SCINet,
  author       = {Minhao Liu and Ailing Zeng and Muxi Chen and Zhijian Xu and Qiuxia Lai and Lingna Ma and Qiang Xu},
  title        = {SCINet: Time Series Modeling and Forecasting with Sample Convolution and Interaction},
  booktitle    = {Proceedings of the Advances in Neural Information Processing Systems (NeurIPS)},
  year         = {2022}
}

@inproceedings{Informer,
  author       = {Haoyi Zhou and Shanghang Zhang and Jieqi Peng and Shuai Zhang and Jianxin Li and Hui Xiong and Wancai Zhang},
  title        = {Informer: Beyond Efficient Transformer for Long Sequence Time-Series Forecasting},
  booktitle    = {Proceedings of the AAAI Conference on Artificial Intelligence (AAAI)},
  pages        = {11106--11115},
  publisher    = {AAAI Press},
  year         = {2021},
  doi          = {10.1609/AAAI.V35I12.17325}
}

@inproceedings{Pyraformer,
  author       = {Shizhan Liu and Hang Yu and Cong Liao and Jianguo Li and Weiyao Lin and Alex X. Liu and Schahram Dustdar},
  title        = {Pyraformer: Low-Complexity Pyramidal Attention for Long-Range Time Series Modeling and Forecasting},
  booktitle    = {Proceedings of the International Conference on Learning Representations (ICLR)},
  publisher    = {OpenReview},
  year         = {2022}
}

@inproceedings{fedformer,
  author       = {Tian Zhou and Ziqing Ma and Qingsong Wen and Xue Wang and Liang Sun and Rong Jin},
  title        = {{FEDformer}: Frequency Enhanced Decomposed Transformer for Long-term Series Forecasting},
  booktitle    = {Proceedings of the International Conference on Machine Learning (ICML)},
  pages        = {27268--27286},
  publisher    = {PMLR},
  year         = {2022}
}

@inproceedings{DLinear,
  author       = {Ailing Zeng and Muxi Chen and Lei Zhang and Qiang Xu},
  title        = {Are Transformers Effective for Time Series Forecasting?},
  booktitle    = {Proceedings of the AAAI Conference on Artificial Intelligence (AAAI)},
  pages        = {11121--11128},
  publisher    = {AAAI Press},
  year         = {2023},
  doi          = {10.1609/AAAI.V37I9.26317}
}

@inproceedings{Adam,
  author       = {Diederik P. Kingma and Jimmy Ba},
  title        = {Adam: {A} Method for Stochastic Optimization},
  booktitle    = {Proceedings of the International Conference on Learning Representations (ICLR)},
  year         = {2015}
}

@inproceedings{RNN1,
  author       = {Syama Sundar Rangapuram and Matthias W. Seeger and Jan Gasthaus and Lorenzo Stella and Yuyang Wang and Tim Januschowski},
  title        = {Deep State Space Models for Time Series Forecasting},
  booktitle    = {Proceedings of the Advances in Neural Information Processing Systems (NeurIPS)},
  pages        = {7796--7805},
  year         = {2018}
}

@inproceedings{TSMixer,
  author       = {Vijay Ekambaram and Arindam Jati and Nam Nguyen and Phanwadee Sinthong and Jayant Kalagnanam},
  title        = {{TSMixer}: Lightweight MLP-Mixer Model for Multivariate Time Series Forecasting},
  booktitle    = {Proceedings of the ACM SIGKDD Conference on Knowledge Discovery and Data Mining (KDD)},
  pages        = {459--469},
  publisher    = {ACM},
  year         = {2023},
  doi          = {10.1145/3580305.3599533}
}

@article{han2024capacity,
  author       = {Lu Han and Han{-}Jia Ye and De{-}Chuan Zhan},
  title        = {The Capacity and Robustness Trade-Off: Revisiting the Channel Independent Strategy for Multivariate Time Series Forecasting},
  journal      = {IEEE Transactions on Knowledge and Data Engineering},
  volume       = {36},
  number       = {11},
  pages        = {7129--7142},
  year         = {2024},
  doi          = {10.1109/TKDE.2024.3400008}
}

@inproceedings{liuitransformer,
  author       = {Yong Liu and Tengge Hu and Haoran Zhang and Haixu Wu and Shiyu Wang and Lintao Ma and Mingsheng Long},
  title        = {{iTransformer}: Inverted Transformers Are Effective for Time Series Forecasting},
  booktitle    = {Proceedings of the International Conference on Learning Representations (ICLR)},
  publisher    = {OpenReview},
  year         = {2024}
}

@inproceedings{wu2021autoformer,
  author       = {Haixu Wu and Jiehui Xu and Jianmin Wang and Mingsheng Long},
  title        = {Autoformer: Decomposition Transformers with Auto-Correlation for Long-Term Series Forecasting},
  booktitle    = {Proceedings of the Advances in Neural Information Processing Systems (NeurIPS)},
  pages        = {22419--22430},
  year         = {2021}
}

@inproceedings{dao2024transformers,
  author       = {Tri Dao and Albert Gu},
  title        = {{Transformers are SSMs}: Generalized Models and Efficient Algorithms Through Structured State Space Duality},
  booktitle    = {Proceedings of the International Conference on Machine Learning (ICML)},
  pages        = {10041--10071},
  publisher    = {PMLR / OpenReview},
  year         = {2024}
}

@inproceedings{han2024demystify,
  author       = {Dongchen Han and Ziyi Wang and Zhuofan Xia and Yizeng Han and Yifan Pu and Chunjiang Ge and Jun Song and Shiji Song and Bo Zheng and Gao Huang},
  title        = {Demystify Mamba in Vision: {A} Linear Attention Perspective},
  booktitle    = {Proceedings of the Advances in Neural Information Processing Systems (NeurIPS)},
  year         = {2024}
}

@inproceedings{liutimer,
  author       = {Yong Liu and Haoran Zhang and Chenyu Li and Xiangdong Huang and Jianmin Wang and Mingsheng Long},
  title        = {Timer: Generative Pre-trained Transformers Are Large Time Series Models},
  booktitle    = {Proceedings of the International Conference on Machine Learning (ICML)},
  pages        = {32369--32399},
  publisher    = {PMLR / OpenReview},
  year         = {2024}
}

@inproceedings{woounified,
  author       = {Gerald Woo and Chenghao Liu and Akshat Kumar and Caiming Xiong and Silvio Savarese and Doyen Sahoo},
  title        = {Unified Training of Universal Time Series Forecasting Transformers},
  booktitle    = {Proceedings of the International Conference on Machine Learning (ICML)},
  pages        = {53140--53164},
  publisher    = {PMLR / OpenReview},
  year         = {2024}
}

@article{qu2024ssd4rec,
  author       = {Yifeng Zhang and Haohao Qu and Liangbo Ning and Wenqi Fan and Qing Li},
  title        = {SSD4Rec: {A} Structured State Space Duality Model for Efficient Sequential Recommendation},
  journal      = {ACM Transactions on Information Systems},
  volume       = {44},
  number       = {2},
  pages        = {29:1--29:26},
  year         = {2026},
  doi          = {10.1145/3773038}
}

@inproceedings{schiff2024caduceus,
  author       = {Yair Schiff and Chia{-}Hsiang Kao and Aaron Gokaslan and Tri Dao and Albert Gu and Volodymyr Kuleshov},
  title        = {Caduceus: Bi-Directional Equivariant Long-Range {DNA} Sequence Modeling},
  booktitle    = {Proceedings of the International Conference on Machine Learning (ICML)},
  pages        = {43632--43648},
  publisher    = {PMLR / OpenReview},
  year         = {2024}
}

@inproceedings{LIFT,
  author       = {Lifan Zhao and Yanyan Shen},
  title        = {Rethinking Channel Dependence for Multivariate Time Series Forecasting: Learning from Leading Indicators},
  booktitle    = {Proceedings of the International Conference on Learning Representations (ICLR)},
  publisher    = {OpenReview},
  year         = {2024}
}

@inproceedings{lieber2024jamba,
  author       = {Barak Lenz and Opher Lieber and Alan Arazi and Amir Bergman and Avshalom Manevich and Barak Peleg and Ben Aviram and Chen Almagor and Clara Fridman and Dan Padnos and Daniel Gissin and Daniel Jannai and Dor Muhlgay and Dor Zimberg and Edden M. Gerber and Elad Dolev and Eran Krakovsky and Erez Safahi and Erez Schwartz and Gal Cohen and et al.},
  title        = {Jamba: Hybrid Transformer-Mamba Language Models},
  booktitle    = {Proceedings of the International Conference on Learning Representations (ICLR)},
  publisher    = {OpenReview},
  year         = {2025}
}

@article{behrouz2024mambamixer,
  author       = {Ali Behrouz and Michele Santacatterina and Ramin Zabih},
  title        = {{MambaMixer}: Efficient Selective State Space Models with Dual Token and Channel Selection},
  journal      = {arXiv preprint arXiv:2403.19888},
  year         = {2024}
}

@article{qu2024survey,
  author       = {Haohao Qu and Liangbo Ning and Rui An and Wenqi Fan and Tyler Derr and Hui Liu and Xin Xu and Qing Li},
  title        = {A Survey of Mamba},
  journal      = {arXiv preprint arXiv:2408.01129},
  year         = {2024}
}

@article{cai2024mambats,
  author       = {Xiuding Cai and Yaoyao Zhu and Xueyao Wang and Yu Yao},
  title        = {{MambaTS}: Improved Selective State Space Models for Long-term Time Series Forecasting},
  journal      = {arXiv preprint arXiv:2405.16440},
  year         = {2024}
}

@inproceedings{vaswani2017attention,
  author       = {Ashish Vaswani and Noam Shazeer and Niki Parmar and Jakob Uszkoreit and Llion Jones and Aidan N. Gomez and Lukasz Kaiser and Illia Polosukhin},
  title        = {Attention is All you Need},
  booktitle    = {Proceedings of the Advances in Neural Information Processing Systems (NeurIPS)},
  pages        = {5998--6008},
  year         = {2017}
}

@inproceedings{long2024unveiling,
  author       = {Qingqing Long and Zheng Fang and Chen Fang and Chong Chen and Pengfei Wang and Yuanchun Zhou},
  title        = {Unveiling Delay Effects in Traffic Forecasting: {A} Perspective from Spatial-Temporal Delay Differential Equations},
  booktitle    = {Proceedings of the ACM Web Conference (WWW)},
  pages        = {1035--1044},
  publisher    = {ACM},
  year         = {2024},
  doi          = {10.1145/3589334.3645688}
}

@inproceedings{han2023flatten,
  author       = {Dongchen Han and Xuran Pan and Yizeng Han and Shiji Song and Gao Huang},
  title        = {FLatten Transformer: Vision Transformer using Focused Linear Attention},
  booktitle    = {Proceedings of the IEEE/CVF International Conference on Computer Vision, ICCV},
  pages        = {5938--5948},
  publisher    = {IEEE},
  year         = {2023},
  doi          = {10.1109/ICCV51070.2023.00548}
}

@article{su2024roformer,
  author       = {Jianlin Su and Murtadha H. M. Ahmed and Yu Lu and Shengfeng Pan and Wen Bo and Yunfeng Liu},
  title        = {RoFormer: Enhanced transformer with Rotary Position Embedding},
  journal      = {Neurocomputing},
  volume       = {568},
  pages        = {127063},
  year         = {2024},
  doi          = {10.1016/J.NEUCOM.2023.127063}
}

@inproceedings{su2019robust,
  author       = {Ya Su and Youjian Zhao and Chenhao Niu and Rong Liu and Wei Sun and Dan Pei},
  title        = {Robust Anomaly Detection for Multivariate Time Series through Stochastic Recurrent Neural Network},
  booktitle    = {Proceedings of the ACM SIGKDD Conference on Knowledge Discovery and Data Mining (KDD)},
  pages        = {2828--2837},
  publisher    = {ACM},
  year         = {2019},
  doi          = {10.1145/3292500.3330672}
}

@inproceedings{abdulaal2021practical,
  author       = {Ahmed Abdulaal and Zhuanghua Liu and Tomer Lancewicki},
  title        = {Practical Approach to Asynchronous Multivariate Time Series Anomaly Detection and Localization},
  booktitle    = {Proceedings of the ACM SIGKDD Conference on Knowledge Discovery and Data Mining (KDD)},
  pages        = {2485--2494},
  publisher    = {ACM},
  year         = {2021},
  doi          = {10.1145/3447548.3467174}
}

@inproceedings{hundman2018detecting,
  author       = {Kyle Hundman and Valentino Constantinou and Christopher Laporte and Ian Colwell and Tom S{\"{o}}derstr{\"{o}}m},
  title        = {Detecting Spacecraft Anomalies Using {LSTMs} and Nonparametric Dynamic Thresholding},
  booktitle    = {Proceedings of the ACM SIGKDD Conference on Knowledge Discovery and Data Mining (KDD)},
  pages        = {387--395},
  publisher    = {ACM},
  year         = {2018},
  doi          = {10.1145/3219819.3219845}
}

@inproceedings{zhu2024vision,
  author       = {Lianghui Zhu and Bencheng Liao and Qian Zhang and Xinlong Wang and Wenyu Liu and Xinggang Wang},
  title        = {Vision Mamba: Efficient Visual Representation Learning with Bidirectional State Space Model},
  booktitle    = {Proceedings of the International Conference on Machine Learning (ICML)},
  pages        = {62429--62442},
  publisher    = {PMLR / OpenReview},
  year         = {2024}
}

@article{bagnall2018uea,
  author       = {Anthony J. Bagnall and Hoang Anh Dau and Jason Lines and Michael Flynn and James Large and Aaron Bostrom and Paul Southam and Eamonn J. Keogh},
  title        = {The {UEA} multivariate time series classification archive, 2018},
  journal      = {arXiv preprint arXiv:1811.00075},
  year         = {2018}
}

@article{wang2024mamba,
  author       = {Zihan Wang and Fanheng Kong and Shi Feng and Ming Wang and Xiaocui Yang and Han Zhao and Daling Wang and Yifei Zhang},
  title        = {Is Mamba effective for time series forecasting?},
  journal      = {Neurocomputing},
  volume       = {619},
  pages        = {129178},
  year         = {2025},
  doi          = {10.1016/J.NEUCOM.2024.129178}
}

@inproceedings{wu2025affirm,
  author       = {Yuhan Wu and Xiyu Meng and Huajin Hu and Junru Zhang and Yabo Dong and Dongming Lu},
  title        = {Affirm: Interactive Mamba with Adaptive Fourier Filters for Long-term Time Series Forecasting},
  booktitle    = {Proceedings of the AAAI Conference on Artificial Intelligence (AAAI)},
  pages        = {21599--21607},
  publisher    = {AAAI Press},
  year         = {2025},
  doi          = {10.1609/AAAI.V39I20.35463}
}

@article{samba,
  author       = {Zixuan Weng and Jindong Han and Wenzhao Jiang and Hao Liu},
  title        = {Simplified Mamba with Disentangled Dependency Encoding for Long-Term Time Series Forecasting},
  journal      = {arXiv preprint arXiv:2408.12068},
  year         = {2024}
}

@article{cmamba,
  author       = {Zeng, Chaolv and Liu, Zhanyu and Zheng, Guanjie and Kong, Linghe},
  title        = {{CMamba}: Channel Correlation Enhanced State Space Models for Multivariate Time Series Forecasting},
  journal      = {arXiv preprint arXiv:2406.05316},
  year         = {2024}
}

@inproceedings{jiang2023pdformer,
  author       = {Jiawei Jiang and Chengkai Han and Wayne Xin Zhao and Jingyuan Wang},
  title        = {{PDFormer}: Propagation Delay-Aware Dynamic Long-Range Transformer for Traffic Flow Prediction},
  booktitle    = {Proceedings of the AAAI Conference on Artificial Intelligence (AAAI)},
  pages        = {4365--4373},
  publisher    = {AAAI Press},
  year         = {2023},
  doi          = {10.1609/AAAI.V37I4.25556}
}

@inproceedings{gu2020hippo,
  author       = {Albert Gu and Tri Dao and Stefano Ermon and Atri Rudra and Christopher R{\'{e}}},
  title        = {HiPPO: Recurrent Memory with Optimal Polynomial Projections},
  booktitle    = {Proceedings of the Advances in Neural Information Processing Systems (NeurIPS)},
  year         = {2020}
}

@article{harris2007parallel,
  author       = {Harris, Mark and Sengupta, Shubhabrata and Owens, John D},
  title        = {Parallel prefix sum (scan) with {CUDA}},
  journal      = {GPU gems},
  volume       = {3},
  number       = {39},
  pages        = {851--876},
  year         = {2007}
}

@inproceedings{gu2022train,
  author       = {Albert Gu and Isys Johnson and Aman Timalsina and Atri Rudra and Christopher R{\'{e}}},
  title        = {How to Train your {HIPPO:} State Space Models with Generalized Orthogonal Basis Projections},
  booktitle    = {Proceedings of the International Conference on Learning Representations (ICLR)},
  publisher    = {OpenReview},
  year         = {2023}
}

@article{kong2025deep,
  author       = {Xiangjie Kong and Zhenghao Chen and Weiyao Liu and Kaili Ning and Lechao Zhang and Syauqie Muhammad Marier and Yichen Liu and Yuhao Chen and Feng Xia},
  title        = {Deep learning for time series forecasting: a survey},
  journal      = {International Journal of Machine Learning and Cybernetics},
  volume       = {16},
  number       = {7-8},
  pages        = {5079--5112},
  year         = {2025},
  doi          = {10.1007/S13042-025-02560-W}
}

@inproceedings{yi2023frequency,
  author       = {Kun Yi and Qi Zhang and Wei Fan and Shoujin Wang and Pengyang Wang and Hui He and Ning An and Defu Lian and Longbing Cao and Zhendong Niu},
  title        = {Frequency-domain MLPs are More Effective Learners in Time Series Forecasting},
  booktitle    = {Proceedings of the Advances in Neural Information Processing Systems (NeurIPS)},
  year         = {2023}
}

@inproceedings{liu2023koopa,
  author       = {Yong Liu and Chenyu Li and Jianmin Wang and Mingsheng Long},
  title        = {Koopa: Learning Non-stationary Time Series Dynamics with Koopman Predictors},
  booktitle    = {Proceedings of the Advances in Neural Information Processing Systems (NeurIPS)},
  year         = {2023}
}

@inproceedings{yue2022ts2vec,
  author       = {Zhihan Yue and Yujing Wang and Juanyong Duan and Tianmeng Yang and Congrui Huang and Yunhai Tong and Bixiong Xu},
  title        = {TS2Vec: Towards Universal Representation of Time Series},
  booktitle    = {Proceedings of the AAAI Conference on Artificial Intelligence (AAAI)},
  pages        = {8980--8987},
  publisher    = {AAAI Press},
  year         = {2022},
  doi          = {10.1609/AAAI.V36I8.20881}
}

@article{azaria1984time,
  author       = {Hertz, D.},
  title        = {Time delay estimation by combining efficient algorithms and generalized cross-correlation methods},
  journal      = {IEEE Transactions on Acoustics, Speech, and Signal Processing},
  volume       = {34},
  number       = {1},
  pages        = {1--7},
  year         = {1986},
  doi          = {10.1109/TASSP.1986.1164789}
}

@article{fang2020time,
  author       = {Chenguang Fang and Chen Wang},
  title        = {Time Series Data Imputation: {A} Survey on Deep Learning Approaches},
  journal      = {arXiv preprint arXiv:2011.11347},
  year         = {2020}
}

@article{gu2023mamba,
  author       = {Albert Gu and Tri Dao},
  title        = {Mamba: Linear-Time Sequence Modeling with Selective State Spaces},
  journal      = {arXiv preprint arXiv:2312.00752},
  year         = {2023}
}

@inproceedings{luo2024moderntcn,
  author       = {Donghao Luo and Xue Wang},
  title        = {{ModernTCN}: {A} Modern Pure Convolution Structure for General Time Series Analysis},
  booktitle    = {Proceedings of the International Conference on Learning Representations (ICLR)},
  publisher    = {OpenReview},
  year         = {2024}
}

@article{an2025damba,
  author       = {Rui An and Yifeng Zhang and Ziran Liang and Wenqi Fan and Yuxuan Liang and Xuequn Shang and Qing Li},
  title        = {{Damba-ST}: Domain-Adaptive Mamba for Efficient Urban Spatio-Temporal Prediction},
  journal      = {arXiv preprint arXiv:2506.18939},
  year         = {2025}
}

@article{liang2025itfkan,
  author       = {Ziran Liang and Rui An and Wenqi Fan and Yanghui Rao and Yuxuan Liang},
  title        = {{iTFKAN}: Interpretable Time Series Forecasting with Kolmogorov-Arnold Network},
  journal      = {arXiv preprint arXiv:2504.16432},
  year         = {2025}
}

@article{tu2026generalized,
  author       = {Jiancheng Tu and Wenqi Fan and Zhibin Wu},
  title        = {Generalized Optimal Classification Trees: {A} Mixed-Integer Programming Approach},
  journal      = {arXiv preprint arXiv:2602.02173},
  year         = {2026}
}

@inproceedings{qu2025diffusion,
  author       = {Haohao Qu and Shanru Lin and Yujuan Ding and Yiqi Wang and Wenqi Fan},
  title        = {Diffusion Generative Recommendation with Continuous Tokens},
  booktitle    = {Proceedings of the ACM Web Conference (WWW)},
  pages        = {7259--7270},
  publisher    = {ACM},
  year         = {2026},
  doi          = {10.1145/3774904.3792428}
}

@article{zhou2025hd,
  author       = {Yi Zhou and Haohao Qu and Yunqing Liu and Shanru Lin and Le Song and Wenqi Fan},
  title        = {{HD-Prot}: {A} Protein Language Model for Joint Sequence-Structure Modeling with Continuous Structure Tokens},
  journal      = {arXiv preprint arXiv:2512.15133},
  year         = {2025}
}

@article{qu2025tokenrec,
  author       = {Haohao Qu and Wenqi Fan and Zihuai Zhao and Qing Li},
  title        = {{TokenRec}: Learning to Tokenize {ID} for LLM-Based Generative Recommendations},
  journal      = {IEEE Transactions on Knowledge and Data Engineering},
  volume       = {37},
  number       = {10},
  pages        = {6216--6231},
  year         = {2025},
  doi          = {10.1109/TKDE.2025.3599265}
}

@article{jia2025principles,
  author       = {Jia, Jian and Gao, Jingtong and Xue, Ben and Wang, Junhao and Cai, Qingpeng and Chen, Quan and Zhao, Xiangyu and Jiang, Peng and Gai, Kun},
  title        = {From principles to applications: A comprehensive survey of discrete tokenizers in generation, comprehension, recommendation, and information retrieval},
  journal      = {arXiv preprint arXiv:2502.12448},
  year         = {2025}
}

@inproceedings{timemixer,
  author       = {Shiyu Wang and Haixu Wu and Xiaoming Shi and Tengge Hu and Huakun Luo and Lintao Ma and James Y. Zhang and Jun Zhou},
  title        = {{TimeMixer}: Decomposable Multiscale Mixing for Time Series Forecasting},
  booktitle    = {Proceedings of the International Conference on Learning Representations (ICLR)},
  publisher    = {OpenReview},
  year         = {2024}
}

@inproceedings{kim2021reversible,
  author       = {Taesung Kim and Jinhee Kim and Yunwon Tae and Cheonbok Park and Jang{-}Ho Choi and Jaegul Choo},
  title        = {Reversible Instance Normalization for Accurate Time-Series Forecasting against Distribution Shift},
  booktitle    = {Proceedings of the International Conference on Learning Representations (ICLR)},
  publisher    = {OpenReview},
  year         = {2022}
}

@article{li2023revisiting,
  author       = {Zhe Li and Shiyi Qi and Yiduo Li and Zenglin Xu},
  title        = {Revisiting Long-term Time Series Forecasting: An Investigation on Linear Mapping},
  journal      = {arXiv preprint arXiv:2305.10721},
  year         = {2023}
}

@inproceedings{liu2024unitime,
  author       = {Xu Liu and Junfeng Hu and Yuan Li and Shizhe Diao and Yuxuan Liang and Bryan Hooi and Roger Zimmermann},
  title        = {UniTime: {A} Language-Empowered Unified Model for Cross-Domain Time Series Forecasting},
  booktitle    = {Proceedings of the ACM Web Conference (WWW)},
  pages        = {4095--4106},
  publisher    = {ACM},
  year         = {2024},
  doi          = {10.1145/3589334.3645434}
}

@inproceedings{kitaev2020reformer,
  author       = {Nikita Kitaev and Lukasz Kaiser and Anselm Levskaya},
  title        = {Reformer: The Efficient Transformer},
  booktitle    = {Proceedings of the International Conference on Learning Representations (ICLR)},
  publisher    = {OpenReview},
  year         = {2020}
}

@inproceedings{li2019enhancing,
  author       = {Shiyang Li and Xiaoyong Jin and Yao Xuan and Xiyou Zhou and Wenhu Chen and Yu{-}Xiang Wang and Xifeng Yan},
  title        = {Enhancing the Locality and Breaking the Memory Bottleneck of Transformer on Time Series Forecasting},
  booktitle    = {Proceedings of the Advances in Neural Information Processing Systems (NeurIPS)},
  pages        = {5244--5254},
  year         = {2019}
}

@inproceedings{LSTNet,
  author    = {Lai, Guokun and Chang, Wei-Cheng and Yang, Yiming and Liu, Hanxiao},
  title     = {Modeling long-and short-term temporal patterns with deep neural networks},
  booktitle = {Proceedings of the ACM SIGIR Conference on Research and Development in Information Retrieval (SIGIR)},
  pages     = {95--104},
  year      = {2018}
}

@article{chen2001freeway,
  title={Freeway performance measurement system: mining loop detector data},
  author={Chen, Chao and Petty, Karl and Skabardonis, Alexander and Varaiya, Pravin and Jia, Zhanfeng},
  journal={Transportation research record},
  volume={1748},
  number={1},
  pages={96--102},
  year={2001}
}

@inproceedings{mathur2016swat,
  author    = {Mathur, Aditya P. and Tippenhauer, Nils Ole},
  title     = {{SWaT}: A water treatment testbed for research and training on {ICS} security},
  booktitle = {Proceedings of the International Workshop on Cyber-Physical Systems for Smart Water Networks ({CySWater})},
  pages     = {31--36},
  year      = {2016}
}

\end{document}